\documentclass{amsart}
\usepackage{amsaddr}
\usepackage[utf8]{inputenc}
\usepackage{amssymb}
\usepackage{amsthm}
\usepackage{amsfonts}
\usepackage{amsmath}
\usepackage{multicol,multirow}
\usepackage[text={440pt,575pt},headheight=9pt,centering]{geometry}

\usepackage{amsmath}
\usepackage{amssymb}
\usepackage{amsthm}
\usepackage{soul}
\usepackage{cancel}
\usepackage{bm}

\usepackage{graphicx}
\usepackage{graphicx}
\usepackage{subcaption}
\usepackage{booktabs}
\usepackage[colorlinks,citecolor=blue,urlcolor=blue]{hyperref}

\usepackage{fullpage}

\usepackage[linesnumbered,ruled,vlined]{algorithm2e}
\SetKwInput{KwInput}{Input}                
\SetKwInput{KwOutput}{Output}              

\usepackage{natbib}

\usepackage[normalem]{ulem}
\usepackage{placeins}

\def\R{\mathbb{R}} 
\def\Id{\mathbb{I}} 

\DeclareMathOperator*{\col}{col}

\DeclareMathOperator*{\rank}{rank}
\DeclareMathOperator*{\spann}{span}
\DeclareMathOperator*{\Tr}{tr}

\newcommand{\Acal}{\mathcal{A}}
\newcommand{\Ncal}{\mathcal{N}}
\newcommand{\Ocal}{\mathcal{O}}
\newcommand{\Scal}{\mathcal{S}}
\newcommand{\Pcal}{\mathcal{P}}
\newcommand{\Jcal}{\mathcal{J}}
\newcommand{\Ical}{\mathcal{I}}

\newtheorem{theorem}{Theorem}[section]

\newtheorem{lemma}[theorem]{Lemma}
\newtheorem{assumption}{Assumption}

\title{A spectral method for multi-view subspace learning using the product of projections}
\author{Renat Sergazinov}
\address{Department of Statistics, Texas A\&M University, College Station, TX}

\author{Armeen Taeb}
\address{Department of Statistics, University of Washington, Seattle, WA}

\author{Irina Gaynanova$^{*}$}
\address{Department of Biostatistics, University of Michigan, Ann Arbor, MI}
\thanks{$^*$Corresponding author: irinagn@umich.edu.}

\begin{document}

\maketitle

\begin{abstract}

Multi-view data provides complementary information on the same set of observations, with multi-omics and multimodal sensor data being common examples. Analyzing such data typically requires distinguishing between shared (joint) and unique (individual) signal subspaces from noisy, high-dimensional measurements. Despite many proposed methods, the conditions for reliably identifying joint and individual subspaces remain unclear. We rigorously quantify these conditions, which depend on the ratio of the signal rank to the ambient dimension, principal angles between true subspaces, and noise levels. Our approach characterizes how spectrum perturbations of the product of projection matrices, derived from each view’s estimated subspaces, affect subspace separation. Using these insights, we provide an easy-to-use and scalable estimation algorithm. In particular, we employ rotational bootstrap and random matrix theory to partition the observed spectrum into joint, individual, and noise subspaces. Diagnostic plots visualize this partitioning, providing practical and interpretable insights into the estimation performance. In simulations, our method estimates joint and individual subspaces more accurately than existing approaches. Applications to multi-omics data from colorectal cancer patients and nutrigenomic study of mice demonstrate improved performance in downstream predictive tasks.
\end{abstract}


\section{Introduction}

Multi-view data provides complementary information on the same set of observations, with multi-omics and multimodal sensor data being common examples. Since each data source or view is potentially high-dimensional, it is of interest to identify a low-dimensional representation of the signal and further distinguish whether the signal is shared (joint) across views or view-specific (individual). For example, in \S~\ref{sec:colorectal}, we consider colorectal cancer data with RNAseq and miRNA views. The goal is to identify joint signals from these two views and validate whether these signals predict cancer subtypes \citep{guinney2015consensus}.

Canonical correlation analysis \citep{hotelling1992relations} is a time-tested approach for finding associated signals across views, but it cannot distinguish individual signals. To address this limitation, many methods have been developed to explicitly model joint and individual signals \citep{van2009structured, lock2013joint, zhou2015group, yang2016non, feng2018angle, shu2019d, gaynanova2019structural, park2020integrative, murden2022interpretive, prothero2024data, xiaoSparseIntegrativePrincipal2024}.

Despite significant methodological advances, a crucial theoretical question remains: when can joint and individual signals be reliably identified from multi-view data? To answer this question, we start with the two-view case and adopt the definition of joint and individual signal subspaces following \citet{lock2013joint, zhou2015group, feng2018angle}. In this framework, the joint signals correspond to the same subspace and are orthogonal to individual signals, whereas the individual signals do not intersect (albeit may not be orthogonal).
Conceptually, it is clear that (1) separation of joint and individual signals becomes difficult when the angle between individual subspaces gets small; (2) the total rank of the signal must be small to avoid noisy directions being mistaken for the joint due to random overlaps. However, the precise quantification of these concepts has been lacking. 
\subsection{Our contributions}
Our main contribution is an explicit characterization of conditions under which the joint and individual signals are identifiable, which depend on the signal rank, noise level, and principal angles between individual subspaces. Our primary insight is that the spectrum perturbations of the \emph{product of projection matrices}, derived from each view’s estimated subspaces, affect subspace separation. Figure~\ref{fig:sing-val-clustering} provides a visual illustration of this theoretical result under varying subspace alignment and noise levels. The observed spectrum of the product of projections, which corresponds to the cosine of principal angles between the two subspaces, exhibits a clustering structure that corresponds exactly to joint and individual subspaces. Specifically, the largest singular values close to 1 correspond to joint subspace, the middle singular values correspond to non-orthogonal individual subspaces, and the smallest singular values correspond to orthogonal individual and noise directions. The bars indicate the observed spectrum, whereas the colored areas highlight the theoretically derived predictions on the spectrum perturbation in alignment with observed clustering. A lack of interval overlap supports the idea that joint subspaces are distinguishable from individual subspaces and noise.

\begin{figure}[!t]
    \centering
    \includegraphics[width=0.9\linewidth]{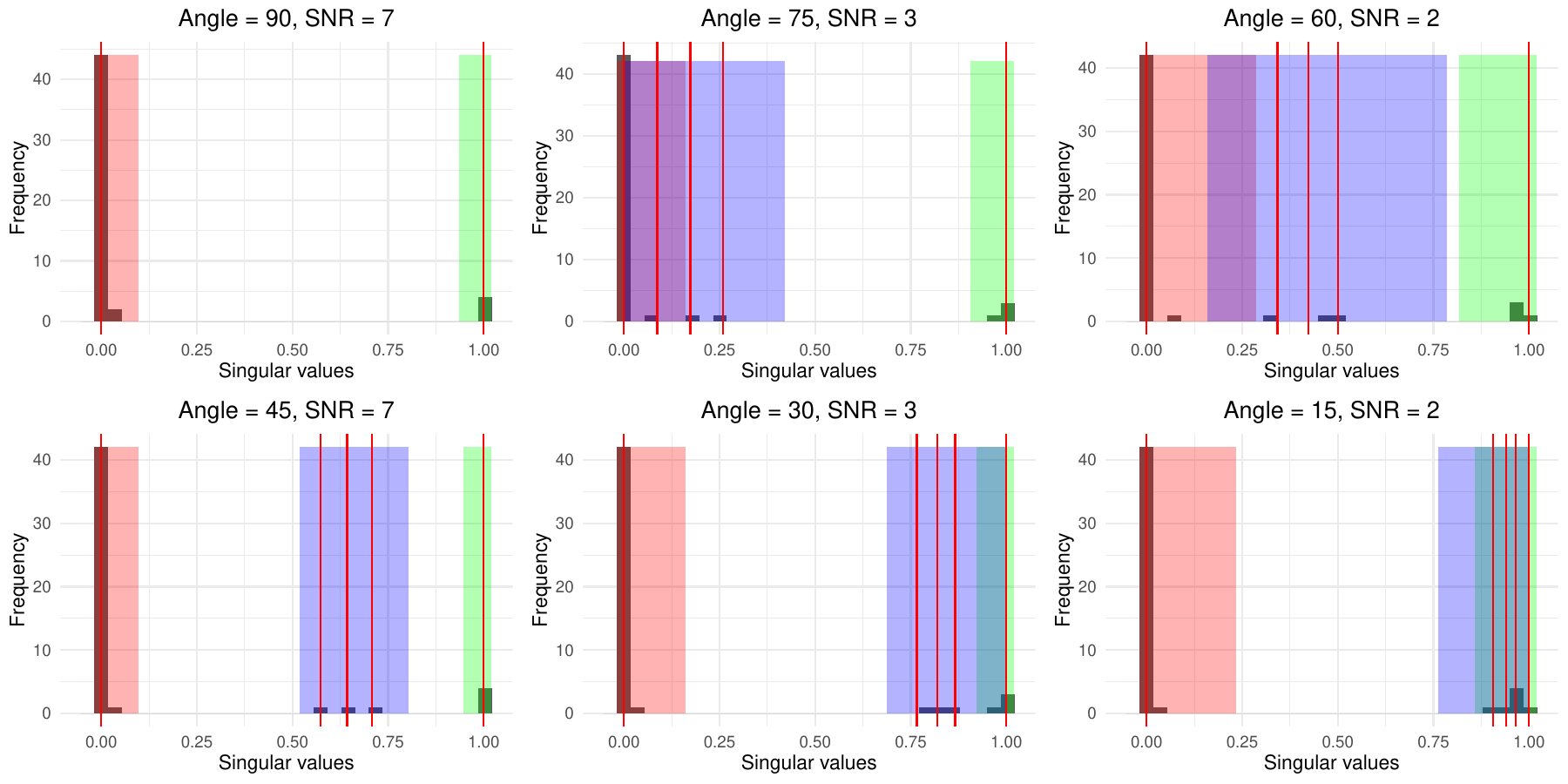}
    \caption{The alignment between the observed spectrum of the product of projections (grey histogram) and theoretical predictions from Theorem~\ref{theorem:main} (highlighted blocks) under varying angles between individual subspaces and signal-to-noise ratios (SNR). The theoretical intervals correspond to the \textcolor{green}{joint}, \textcolor{blue}{non-orthogonal individual}, and \textcolor{red}{orthogonal individual and noise} components. The vertical \textcolor{red}{red} lines show the singular values in the noiseless setting. 
    }
    \label{fig:sing-val-clustering}
\end{figure}

Motivated by new theoretical insights on the perturbation of the spectrum of the product of projections, we propose an estimation approach that quantifies these perturbations in practice with an easy-to-use and scalable algorithm. First, we develop a novel rotational bootstrap technique to bound the perturbations of the spectrum corresponding to joint and non-orthogonal individual subspaces (top two clusters). Second, we utilize results from random matrix theory to quantify the maximal alignment possible by chance, providing a bound on the bottom cluster (orthogonal individual and noisy subspace directions). We use these two bounds to obtain estimates for the joint and individual subspaces. Further, we generate diagnostic plots that visualize cluster partitioning with the two bounds overlaid (like the one in Figure~\ref{fig:sing-val-clustering}), providing practical and interpretable insights regarding the identifiability of joint and individual subspaces. 

In simulations, our method is competitive compared to those of \citet{lock2013joint,feng2018angle,gaynanova2019structural,shu2019d,park2020integrative}, particularly in terms of true and false discovery rates in estimating joint and individual subspaces. Additionally, we test our approach on colorectal cancer data and nutrigenomic data. The colorectal cancer data includes tumor samples with RNASeq and miRNA views, where we aim to identify joint components. We assess the quality of these components through their predictive power for cancer subtype classification. The nutrigenomic data consists of mice with distinct genotypes and diets, with RNASeq and liver lipid content as the two views. Here, we seek to identify joint and individual components and link them to genotype and diet labels, respectively. In both datasets, our approach achieves the best balance between parsimony and predictive accuracy.



\subsection{Related works}
 \citet{lock2013joint, zhou2015group, gaynanova2019structural, park2020integrative, yi2023hierarchical, prothero2024data, xiaoSparseIntegrativePrincipal2024} consider optimization-based approaches for estimating joint and individual subspaces. While effective, these methods can be computationally slow, particularly with high-dimensional data, and often rely on ad-hoc techniques for rank selection (such as permutation or bi-cross-validation). {\citet{xiaoSparseIntegrativePrincipal2024} is an exception in that it provides a rigorous theoretical analysis of estimation consistency for its optimization-based approach. However, its methodological framework is based on sparse principal component analysis and on estimating eigenvectors corresponding to features. In contrast, our work focuses on joint and individual subspaces, corresponding to samples, based on spectral methods. }\citet{feng2018angle}, \citet{shu2019d} { and \citet{chenTwostageLinkedComponent2022}} are spectral methods, like ours, and utilize singular value decomposition to derive joint and individual subspaces. \citet{shu2019d} extends canonical correlation analysis to account for individual components. However, the individual components are restricted to be orthogonal, leading to a different model from ours. \citet{feng2018angle} uses the spectrum of the average of projection matrices, which contains information about the principle angles between the two subspace views, to estimate joint and individual components. The method allows for non-orthogonal individual signals and is most similar to ours; however, it has certain limitations in comparison, as we discuss below. { \citet{chenTwostageLinkedComponent2022} consider a different data integration setting where data sets are linked by features rather than matched by samples, using a 2 $\times$ 2 factorial design. Similar to \citet{feng2018angle}, their method relies on the spectrum of the average of projection matrices for estimation.}

Conceptually, \citet{feng2018angle} { and \citet{chenTwostageLinkedComponent2022} consider} the average of projection matrices, while our method is the first to focus on the product of projection matrices. A key disadvantage of using the average of projection matrices is that its spectrum exhibits a provably smaller gap (by a factor of less than one-half in noiseless settings) between individual and joint components than the product. This results in a substantially worse clustering of singular values in practice, especially in moderate to low signal-to-noise regimes (see e.g, Figure~\ref{fig:ppd-diagnostic-simulations}
 in Appendix~\ref{sec:comparison_ajive}). The lack of clustering in the spectrum considered by \citet{feng2018angle} makes separating the joint directions from the remaining signals more challenging. See \S~\ref{sec:simulation} for more numerical comparison and Appendix~\ref{sec:comparison_ajive} for more technical details.
 
\section{Model}

\subsection{Notation}\label{sec:notation}

For $a, b \in \mathbb{R}$, let $a \wedge b = \min(a, b)$, $a \vee b = \max(a, b)$. For a matrix $A \in \mathbb{R}^{n \times p}$, write the compact singular value decomposition (SVD) as $A = U \Sigma V^T$, where $\Sigma = \text{diag}(\sigma_1(A), \sigma_2(A), \dots)$ with the singular values $\sigma_1(A) \geq \sigma_2(A) \geq \dots > 0$ in descending order. We denote $\sigma_{\min}(A) = \sigma_{\rank(A)}(A)$, $\sigma_{\max}(A) = \sigma_1(A)$ as the smallest and largest non-trivial singular values. We denote $\sigma(A)$ to be the collection $\{\sigma_1(A),\sigma_2(A), \dots,\sigma_{\rank(A)}(A)\}$. We use $\|A\|_2 = \sigma_1(A)$ to denote the spectral norm, $\|A\|_F$ to denote the Frobenius norm and $\col(A)$ to denote the column space of $A$. We write the subspaces in calligraphic, e.g., we say $\Acal \subset \R^{n}$ to denote a subspace of $\R^n$. For the projection onto a subspace $\Acal$, we write $P_{\Acal} \in \mathbb{R}^{n \times n}$. For the projection onto the column space of $A$, we use $P_{\col(A)}$, which can be defined as $P_A = U U^\top$. We denote the direct sum of two subspaces $\mathcal{A}$ and $\mathcal{B}$ by $\mathcal{A} \oplus \mathcal{B}$. We adopt the convention to denote the submatrix: $A_{[a:b,c:d]}$ represents the $a$-to-$b$-th row, $c$-to-$d$-th column of matrix $A$; we also use $A_{[a:b,:]}$ and $A_{[:,c:d]}$ to represent $a$-to-$b$-th full rows of $A$ and $c$-to-$d$-th full columns of $A$, respectively. { We use the
Bachman-Landau symbols $\mathcal{O}$ to describe the limiting behavior of a function (in terms of sample size and number of features). Further, we use $z \asymp 1$ to express $z = \mathcal{O}(1)$ and $1/z = \mathcal{O}(1)$. Finally, we use $z_1 \gtrsim z_2$ to mean that $z_1 \geq {c}z_2$ for some constant $c>0$.} 

\subsection{Setup}
Let $Y_1\in \R^{n \times p_1}, Y_2\in \R^{n \times p_2}$ be observed data matrices corresponding to two views. The rows of both matrices represent common $n$ subjects, and the columns correspond to the different feature sets collected for each view. For example, in the COAD data (see \S~\ref{sec:colorectal}), the first view corresponds to the RNASeq features and the second to the miRNA features. In the nutrigenomic study (see \S~\ref{sec:nutrigenomic}), the first view corresponds to the gene expressions and the second view corresponds to the lipid content of the liver. We assume that $Y_1$, $Y_2$ are noisy observations of true $X_1, X_2$, and consider additive structured signal plus noise decomposition:
\begin{align}\label{eq:additive}
    Y_k &= X_k + Z_k = J_k + I_k + Z_k,
\end{align}
where the signal $X_k$ is decomposed into joint part across views ($J_k$) and the view-specific individual part ($I_k$), and $Z_k$ is the noise matrix with elements identically and independently distributed according to some distribution. 

Model~\eqref{eq:additive} is adopted by many existing methods for multi-view data integration \citep{lock2013joint, feng2018angle, gaynanova2019structural, shu2019d, park2020integrative, yi2023hierarchical,prothero2024data} with some variations in underlying definitions of joint and individual signals. Here we adopt the original definition as in \citet{lock2013joint}. To facilitate exposition, consider singular value decomposition representation of each part of the signal:
\begin{align}
X_k &= J_k + I_k = U_{\Jcal_k} \Sigma_{\Jcal_k} V^\top_{\Jcal_k} +  U_{\Ical_k} \Sigma_{\Ical_k} V^\top_{\Ical_k}.
\end{align}
We refer to $U_{\Jcal_k}$ as joint directions and to $U_{\Ical_k}$ as individual directions.
Joint directions are joint due to the equality of corresponding subspaces, that is they correspond to the joint subspace $\Jcal := \col(U_{\Jcal_1}) = \col(U_{\Jcal_2})$. Thus, joint $U_{\Jcal_1}$ can be reconstructed from joint $U_{\Jcal_2}$ by a rotation by some orthogonal matrix. In contrast, individual directions have non-intersecting subspaces $\Ical_k := \col(U_{\Ical_k})$, that is $\Ical_1 \cap \Ical_2 = \{\mathbf{0}\}$, which emphasizes that they are indeed specific to each view. For identifiability, joint and individual subspaces are assumed to be orthogonal, that is $\Jcal \perp \Ical_k$.

However, the individual $\Ical_1$ and $\Ical_2$ are not necessarily orthogonal and can be further decomposed into orthogonal and non-orthogonal parts, leading to a three-part decomposition (joint, non-orthogonal individual, orthogonal individual). We further provide explicit characterization and identifiability conditions of this three-part decomposition of signal column spaces, which is based on Lemma 1 of \citet{feng2018angle} and Proposition 4 of \citet{gaynanova2019structural}.

\begin{lemma}[Existence and Identifiability] 
\label{lemma:identifiability}
Given a set of subspaces $\{\col(X_1), \col(X_2)\}$, there is a unique set of subspaces $\{\Jcal, \Ncal_1, \Ncal_2, \Ocal_1, \Ocal_2\}$ such that:
\begin{enumerate}
    \item $\col(X_k) = \Jcal \oplus \Ncal_k \oplus\Ocal_k$ with individual $\Ical_k := \Ncal_k \oplus \Ocal_k$;
    \item $\Jcal \perp \Ical_k$ { for all $k\in\{1, 2\}$}, $\Ncal_k \perp \Ocal_l$ for { all $k, l \in \{1, 2\}$};
    \item $\Ocal_1 \perp \Ocal_2$;
    \item $\bigcap \Ncal_k = \{\mathbf{0}\}$ and maximal principle angle between $\Ncal_1$ and $\Ncal_2$ is $<\pi/2$.
\end{enumerate}
\end{lemma}

Here, the individual subspace is given by $\Ical_k = \Ncal_k \oplus \Ocal_k$, with $\Ncal_k$ corresponding to non-orthogonal part and $\Ocal_k$ to orthogonal part. 
The three-part decomposition can be viewed from the perspective of ordering the principal angles between subspaces $\col(X_1)$ and $\col(X_2)$ (\S~\ref{sec:noiseless}). Specifically, it can be viewed as aligning the basis of $\col(X_1)$ and $\col(X_1)$ into pairs, from perfectly aligned pairs (zero angles, joint subspace $\Jcal$) to somewhat aligned pairs (angles strictly between 0 and $\pi/2$, non-orthogonal $\Ncal_1$ and $
\Ncal_2$) to orthogonal basis that can not be paired ($\pi/2$ angles, orthogonal $\Ocal_1$ and $\Ocal_2$).

Our main goal is to distinguish joint subspace from individual based on the noisy observed $Y_k$. Conceptually, it is clear that (1) separation of joint and individual
signals becomes difficult when the angle between individual subspaces gets small, i.e. the angle between $\mathcal{N}_1$ and $\mathcal{N}_2$; (2) the total
rank of the signal must be small to avoid noisy directions being mistaken for the joint signal due to
random overlaps. However, the precise quantification of these concepts has been lacking. Next, we illustrate how the spectrum of the product of projection matrices allows us to distinguish joint subspaces from the aligned individual ones and provide theoretical quantification of the spectrum perturbation as a result of noise. We then use these insights to develop an effective and interpretable estimation approach for joint and individual subspaces.

\section{{A New Perspective: Product of Projections Decomposition}}\label{sec:theory}
Our theoretical insights, which subsequently motivate the proposed estimation approach, rely crucially on the product of projection matrices associated with each view's subspace. The spectrum of this product corresponds to the cosines of principal angles between the pair of subspaces \citep{Bjoerck1971NumericalMF}. In noiseless cases, this spectrum can be clustered to perfectly separate joint and individual subspaces (\S~\ref{sec:noiseless}). In the presence of noise, this spectrum is perturbed, and in \S~\ref{sec:noisy}, we quantify these perturbations with respect to each cluster, leading to explicit conditions for correct estimation of joint rank and quantification of corresponding subspace estimation error. Throughout, we focus on the case of two views, with Appendix~\ref{sec:multiview} describing the extension to the larger number of views. 

\subsection{Product of projections in the noiseless setting}\label{sec:noiseless}
    Consider the noiseless case with $Z_k = \mathbf{0}_{n\times p_k}$ so that $Y_k = X_k$. Let $X_k = U_k \Sigma_k V_k^\top$ be the compact singular value decomposition of signal $X_k$. Using Lemma~\ref{lemma:identifiability}, we have:
    \begin{align*}
    P_{\col(X_k)} = U_k U_k^\top = P_{\Jcal} + P_{\Ical_k}
    \quad\Rightarrow \quad P_{\col (X_{1})} P_{\col (X_2)} = P_{\Jcal} + P_{\Ical_1}P_{\Ical_2}.
    \end{align*}
    Consider the second term $P_{\Ical_1}P_{\Ical_2}$. Let $U_{\Jcal} \in \R^{n\times \dim(\Jcal)}$ be the matrix whose columns form orthonormal basis of $\Jcal$, similarly let $U_{\Ncal_1}, U_{\Ncal_2}, U_{\Ocal_1}, U_{\Ocal_2}$ be the matrices of the orthonormal bases for $\Ncal_1, \Ncal_2, \Ocal_1, \Ocal_2$.  Then,
    \begin{align*}
        P_{\Ical_1} P_{\Ical_2} = (P_{\Ncal_1} + P_{\Ocal_1})(P_{\Ncal_2} + P_{\Ocal_2})
        = P_{\Ncal_1} P_{\Ncal_2}
        = U_{\Ncal_1} {U_{\Ncal_1}^\top U_{\Ncal_2}}{ U_{\Ncal_2}^{\top}}
        = \left(U_{\Ncal_1} H \right)\Sigma \left(K^\top U_{\Ncal_2}^\top \right).
    \end{align*}
   Here, $H\Sigma{K}^{T}$ is the singular value decomposition of $U_{\Ncal_1}^\top U_{\Ncal_2}$; diagonal entries of $\Sigma$ equal cosines of the non-zero principal angles between $\Ical_1$ and $\Ical_2$. Since $\Ical_1 \cap \Ical_2 = \{\mathbf{0}\}$, $\sigma_{\max}(\Sigma) < 1$. Thus,
    $$
    P_{\col (X_1)} P_{\col (X_2)} = \left [U_{\Jcal}\quad U_{\Ncal_1} H \right] \begin{bmatrix}
        I_{\dim \Jcal} & 0 \\
        0 & \Sigma 
    \end{bmatrix} \left [U_{\Jcal}\quad U_{\Ncal_2} K\right]^\top,
    $$
    where the matrices $\left[U_{\mathcal{J}}\quad U_{\Ncal_1} H\right]$ and $\left[U_{\mathcal{J}}\quad U_{\Ncal_2} K\right]$ { have orthonormal columns}. Thus, the spectrum of $P_{\col (X_1)} P_{\col (X_2)}$ clusters into two components: a cluster with singular values equal to one and a cluster with singular values in the range $[\sigma_{\min}(\Sigma),\sigma_{\max}(\Sigma)]$. Let $r_{\text{clust1}}$ be the number of singular values in the first cluster (corresponding to singular values one). Then, the joint and individual subspaces can be exactly identified via:
    \begin{eqnarray*}
    \begin{aligned}
        \mathcal{J} &= \text{ span of the first }r_{\text{clust1}} \text{ left (or right) singular vectors of }P_{\col (X_1)} P_{\col (X_2)},\\
        \mathcal{I}_k &= \text{ span of the first }\rank(X_k)-r_{\text{clust1}} \text{ left (or right) singular vectors of }P_{\col (X_k)}(I-P_{\Jcal}).
        \end{aligned}
    \end{eqnarray*}

Thus, in the noiseless setting, the spectrum of the product of projections $ P_{\col (X_1)} P_{\col (X_2)}$ (equivalently, the set of cosines of principle angles between $\col (X_1)$ and $\col (X_2)$) can be divided into two groups. The first group contains values equal to one (equivalently, principle angles $0$), whose corresponding singular vectors form a basis for the joint subspace. The second group contains values strictly less than one (equivalently, principle angles strictly greater than $0$), reflecting the degree of alignment between the individual subspaces $\Ical_1$ and $\Ical_2$. The gap between these two groups, formally defined as one minus the cosine of the smallest principle angle between $\Ical_1$ and $\Ical_2$, narrows as the alignment between individual subspaces increases. 

\subsection{Deterministic analysis of the product of projection matrices in a noisy setting}\label{sec:noisy}
In practice, we do not have direct access to the subspace $\col(X_k)$ because the signal $X_k$ is unknown. Instead, we have to estimate $\widehat{X}_k$ from noisy observations, with a truncated singular value decomposition of $Y_k$ being a common approach \citep{feng2018angle, shu2019d}. However, the spectrum of the estimated product $P_{\col (\widehat{X}_1)} P_{\col (\widehat{X}_2)}$ differs from the spectrum of the true product $P_{\col (X_1)} P_{\col (X_2)}$ in several key ways.  First, even when a joint signal is present, the largest singular values in the estimated product are generally below one and decrease further as the noise level or the dimension of the individual views increases. Second, in the estimated product, the gap between the singular values corresponding to the joint subspaces and those corresponding to the aligned individual subspaces is typically smaller than in the true product. This gap also decreases as noise or dimensionality increases, making it harder to separate joint signals from individual ones, especially when the individual signals are closely aligned. Finally, even when the individual signals are orthogonal, the estimated product may still have small but non-zero singular values, with their magnitudes increasing as the noise level rises.

One of our main theoretical contributions lies in the precise characterization of perturbations in the spectrum of $\widehat{M} := P_{\col (\widehat{X}_1)} P_{\col (\widehat{X}_2)}$. Our result provides general guarantees that hold regardless of the estimation procedure used to approximate the signals $(\widehat{X}_1,\widehat{X}_2)$. The quality of the chosen estimation procedure is reflected in our analysis via deterministic quantities $\varepsilon_1 := \left\|P_{{\col (X_1)}}\left(\Delta_1+\Delta_2+\Delta_1\Delta_2\right)P_{{\col (X_2)}}\right\|_2$ and $\varepsilon_2 := \left\|P_{{\col (X_1)}}\Delta_2 + \Delta_1P_{{\col (X_2)}}+\Delta_1\Delta_2\right\|_2$ with $\Delta_k := P_{{\col (X_k)}} - P_{\col (\widehat{X}_k)}$. The proof is in the Appendix~\ref{sec:proofs}.

\begin{theorem}[Spectrum of the product of projections] 
\label{theorem:main}
Let $\tau_{\min} := \sigma_{\min}(P_{\Ncal_1} P_{\Ncal_2})$ and $\tau_{\max} := \sigma_{\max}(P_{\Ncal_1} P_{\Ncal_2}) = \left\|P_{\Ncal_1} P_{\Ncal_2}\right\|_2$. If $\rank(X_1), \rank(X_2) > 0$, then the spectrum of $\widehat{M}$ is:
    \begin{enumerate}
        \item a group of $\dim(\mathcal{J})$ singular values in $\left[1 - \varepsilon_1 \vee 0, 1\right];$
        \item a group of $\mathrm{rank}\left(P_{\Ncal_1} P_{\Ncal_2} \right)$  singular values in $\left[\left(\tau_{\min} - \varepsilon_1\right) \vee 0, \left(\tau_{\max} +\varepsilon_2 \right)\wedge 1\right];$
        \item a remaining group of singular values in $[0, \varepsilon_2 \wedge 1]$.
    \end{enumerate}
 \noindent
\end{theorem}
This theorem states that in the presence of a signal, the spectrum of the product of projection matrices $\widehat{M}$ organizes into three groups: joint, non-orthogonal individual, and noise directions. { Here $\varepsilon_1$ quantifies the maximal downward perturbation of the singular values associated with the joint signal—i.e., how much true singular values of one may shrink toward zero due to noise. These perturbed singular values fall within $[1 - \varepsilon_1 \vee 0, 1]$. In contrast, $\varepsilon_2$ captures the maximal upward perturbation of null (noise-only) singular values, which may become spuriously large due to random overlap; these lie in $[0, \varepsilon_2 \vee 1]$. When these intervals are disjoint, joint and noise directions can be perfectly separated. More specifically,} the constraint $\varepsilon_1 < \max\{1-\varepsilon_2, 1-\left\|P_{\Ncal_1} P_{\Ncal_2}\right\|_2-\varepsilon_2\}$ ensures that the singular values corresponding to the joint structure are distinguishable from other clusters, that is the number of singular values above $1-\varepsilon_1$ correctly identifies the joint rank. Similarly, the constraint $\varepsilon_1+\varepsilon_2 < \sigma_{\min}(P_{\Ncal_1} P_{\Ncal_2})$ ensures that the group of singular values corresponding to the non-orthogonal individual structure is distinguishable from the noise group. When both constraints hold, the spectrum of $\widehat{M}$ has \emph{three non-overlapping clusters} corresponding to joint, non-orthogonal individual, and noise directions. The gap between the clusters decreases with larger values of $\varepsilon_1,\varepsilon_2$ and larger alignment of the non-orthogonal individual subspaces ${\Ncal_1}$ and ${\Ncal_2}$. Figure~\ref{fig:sing-val-clustering} illustrates these theoretical intervals alongside observed singular values of $\widehat{M}$ and true singular values of $P_{\col (X_1)} P_{\col (X_2)}$.

The error terms $\varepsilon_1$, $\varepsilon_2$ depend on the noise level: higher noise leads to larger perturbations $\Delta_k$ in estimating each view, and thus larger values of $\varepsilon_1,\varepsilon_2$. Figure~\ref{fig:sing-val-clustering} shows how a smaller signal-to-noise ratio results in smaller cluster gaps. The estimated subspace dimension $\col(\widehat{X}_k)$ also affects the size of $\varepsilon_1$, $\varepsilon_2$. Underestimating the rank (e.g.,  $\rank(\widehat{X}_1) <\rank(X_1)$) can lead to large $\varepsilon_1$ and $\varepsilon_2$, especially if missing directions are aligned with both $\col(X_1)$ and $\col(X_2)$. In fact, $\varepsilon_1 = 1$ when a missing direction is part of the joint structure, in which cases the effective number of large singular values corresponding to the joint structure gets smaller than the true joint rank. Overestimating the rank (e.g.,  $\rank(\widehat{X}_1) > \rank(X_1)$) can also lead to large errors due to increased dimensionality of $\Delta_k$, making the third component of $\varepsilon_2$, $\|\Delta_1\Delta_2\|_2$, very large, albeit slight overestimation would still keep those terms small. Thus, overestimation is generally preferred to underestimation. We use these insights in \S~\ref{sec:estimation}, providing practical guidelines on marginal rank estimation.

The clustering of the singular values of $\widehat{M}$ enables a natural estimator for the signal subspaces. The joint subspace dimension is estimated as the number of singular values of $\widehat{M}$ above $1-\varepsilon_1$, denoted by $\widehat{r}_{\Jcal}$. According to Theorem~\ref{theorem:main}, this estimate is exact if $\varepsilon_1 < 1-\left\|P_{\Ncal_1} P_{\Ncal_2}\right\|_2-\varepsilon_2$. To estimate the joint subspace itself, a candidate approach is to use the span of the first $\widehat{r}_{\Jcal}$ left (or right) singular vectors of $\widehat{M}$. However, this approach arbitrarily favors one view over the other, insisting that the joint subspace is strictly contained in $\col (\widehat{X}_1)$. To avoid this, we symmetrize the product of projection matrices by $\widehat{S} := \frac{1}{2}(\widehat{M} + \widehat{M}^\top)$, which effectively averages the estimated joint directions from each view, yielding a more accurate estimate. We estimate the joint and individual subspace structures as
    \begin{eqnarray}
    \begin{aligned}
        \widehat{\Jcal} &:= \text{span of the first } \widehat{r}_{\Jcal} \text{ singular vectors of }\widehat{S} = \frac12\left(P_{\col (\widehat{X}_1)} P_{\col (\widehat{X}_2)} + P_{\col (\widehat{X}_2)}P_{\col (\widehat{X}_1)} \right),\\
        \widehat{\Ical}_k &:= \text{span of the first } \rank(\widehat{X}_k)-\widehat{r}_{\Jcal} \text{ left singular vectors of }P_{\col(\widehat{X}_k)}(I-P_{\widehat{\Jcal}}).
        \end{aligned}
        \label{eqn:estimated_subspaces}
    \end{eqnarray}
{ We use the symmetrized $\widehat{S}$ to be invariant to the order of views with coinciding left and right singular vectors.}
We further characterize the quality of these estimates, the proof is in Appendix~\ref{sec:proofs}. 
\begin{theorem}[Estimation Error of the Joint and Individual Subspaces]
\label{thm:subspace_est}
Define ${R}_{\Jcal} := P_{\col(X_1)} \Delta_2 + \Delta_1 P_{\col(X_2)} + \Delta_1 \Delta_2$ and  {$R_{\Ical_k} := P_{\Jcal^\perp} \Delta_k - \Delta_{\Jcal} P_{\col(X_k)}+ \Delta_{\Jcal} \Delta_k$}, where {$\Delta_{\Jcal} :=  P_{\Jcal} - P_{\widehat{\Jcal}}$} and $\Delta_k := P_{{\col (X_k)}} - P_{\col (\widehat{X}_k)}$. Suppose $\varepsilon_1 < 1-\left\|P_{\Ncal_1} P_{\Ncal_2}\right\|_2-\varepsilon_2$. Then, 
$$\dim(\widehat{\Jcal}) = \dim(\Jcal) \quad \text{and}\quad\left\| P_{\Jcal} - P_{\widehat{\Jcal}} \right\|_2 \leq  \frac{ {} \left\|R_{\Jcal} + R^\top_{\Jcal}\right\|_2}{1-\left\|P_{\Ncal_1} P_{\Ncal_2}\right\|_2}.$$ Further, if the rank of each subspace view is correctly estimated, i.e. $\rank(\widehat{X}_k) = \rank(X_k)$,
$$\left\|P_{\Ical_k} - P_{\widehat{\Ical}_k}\right\|_2 \leq {2}\left\|R_{\Ical_k}\right\|_2.$$
\end{theorem}
This theorem quantifies the accuracy of estimated subspaces when the ranks are correctly specified. Since the spectral norm of the difference between two projection matrices equals the sine of the largest principle between the corresponding subspaces, our bounds measure the sine of the largest principle angles between  $\Jcal$ and $\widehat{\Jcal}$, and between ${\Ical}_k$ and $\widehat{\Ical}_k$ , respectively. For joint subspace estimation, the bound depends on the perturbation term $R_{\Jcal}$ and the maximum alignment $\left\|P_{\Ncal_1} P_{\Ncal_2}\right\|_2$ between the individual subspaces. {As in Theorem~\ref{theorem:main}, large alignment between the individual subspaces can lead to significant errors. Note that the error due to such alignment also affects individual subspace estimation through the $\Delta_{\Jcal}$ term in $R_{\Ical_k}$.} 

{
\subsection{Probabilistic bounds}\label{sec:noisy_statistical}
Theorems~\ref{theorem:main}-\ref{thm:subspace_est} are deterministic and hold for any signal estimate $\widehat{X}_k$ and for any data-generating noise $Z_k$ in \eqref{eq:additive}. In this section, we provide probabilistic bounds for a particular signal estimation procedure and noise model. Specifically, we let $\widehat{X}_k$ denote the rank-$\widehat{r}_k$ truncated singular value decomposition of $Y_k$, and consider Gaussian noise.
\begin{assumption}(Gaussian noise)\label{a:noise}
The entries of $Z_k$ in~\eqref{eq:additive} are independently and identically distributed as zero-mean Gaussian random variables with variance $\delta_k$, $k=1,2$.
\end{assumption}

\begin{assumption}(Signal-to-noise ratio)\label{a:signal_noise_ratio} For each $k = 1,2$, $\sigma_{r_k}(X_k) \gtrsim \delta_k(n^{1/2}+p_k^{1/2})$. That is, the smallest non-zero singular value of each view's signal is sufficiently large.
\end{assumption}

Under Assumptions~\ref{a:noise} and \ref{a:signal_noise_ratio}, along with rank estimation conditions, we obtain the following probabilistic bounds in terms of the essential aspects of the problem (e.g., the sample size, the number of features in each view, signal ranks) and omit constants. The proof is provided in Appendix~\ref{proof:probabilistic}.
\begin{theorem}[Probabilistic bounds] Suppose Assumptions~\ref{a:noise}-\ref{a:signal_noise_ratio} hold. Let $r_{\max} := \max\{r_1,r_2\}$,  $p_{\max}:= \max\{p_1,p_2\}$, $p_{\min}:= \min\{p_1,p_2\}$, $\delta_{\max}:= \max\{\delta_1,\delta_2\}$, and $\sigma_r(X):= \min\{\sigma_{r_1}(X_1),\sigma_{r_2}(X_2)\}$. If $\widehat r_k \geq r_k$ for $k = 1,2$, then, with probability greater than $1-\mathcal{O}(\exp\{{-\max(n,p_{\min})}\})$:
$$\epsilon_1\lesssim \frac{\delta_{\max}(n^{1/2}+r_{\max}^{1/2})}{\sigma_{r}(X)} + \frac{\delta_{\max}^2(n^{1/2}+p_{\max}^{1/2})(p_{\max}^{1/2}+r_{\max}^{1/2})}{\sigma_{r}(X)^2}.$$
If $\widehat r_k = r_k$  then with probability greater than $1-\mathcal{O}(\exp\{-\max(n,p_{\min})\})$:
$$
\epsilon_2 \lesssim  \frac{\delta_{\max}(n^{1/2}+r_{\max}^{1/2})}{\sigma_{r}(X)} + \frac{\delta_{\max}^2(n^{1/2}+p_{\max}^{1/2})(p_{\max}^{1/2}+r_{\max}^{1/2})}{\sigma_{r}(X)^2}.$$
If additionally $\widehat r_J = r_J$, then with probability greater than $1-\mathcal{O}(\exp\{-\max(n,p_{\min})\})$:
\begin{align*}
&\max\left\{\left\|\mathcal{P}_{\widehat{\Jcal}}-\mathcal{P}_{{\Jcal}}\right\|_2,\left\|\mathcal{P}_{\widehat{\Ical}_k}-\mathcal{P}_{{\Ical}_k}\right\|_2\right\}\\
&\quad\quad\quad\quad \lesssim  
\frac{1}{1-\left\|P_{\Ncal_1} P_{\Ncal_2}\right\|_2}\left\{ \frac{\delta_{\max}(n^{1/2}+r_{\max}^{1/2})}{\sigma_{r}(X)} + \frac{\delta_{\max}^2(n^{1/2}+p_{\max}^{1/2})(p^{1/2}+r_{\max}^{1/2})}{\sigma_{r}(X)^2}\right\}.
\end{align*}

\label{thm:probabilistic}
\end{theorem}

Theorem~\ref{thm:probabilistic} makes explicit the distinctive effect of rank estimation on $\varepsilon_1$ and $\varepsilon_2$. Although we obtained identical bounds, the bound for $\varepsilon_1$ holds true for all $\widehat r_k \geq r_k$, whereas the bound for $\varepsilon_2$ is only true when $\widehat r_k = r_k$. Intuitively, rank overestimation does not adversely affect $\varepsilon_1$, as the alignment of joint directions is already captured. In contrast, adding extra noise dimensions increases the chance of a random strong alignment, pushing $\varepsilon_2$ to be high.

Assumption~\ref{a:signal_noise_ratio} is a signal-to-noise ratio assumption that is common in the random matrix theory due to known concentration $\|Z_k\|_2  \asymp\delta_k(n^{1/2}+p_k^{1/2})$ \citep{vershyninHighDimensionalProbabilityIntroduction2025}, with high probability. This leads to the standard bound $\|\Delta_k\|_2 \lesssim  \delta_k(n^{1/2} + p_k^{1/2})/\sigma_{r_k}(X_k)$ via Davis–Kahan theorem \citep{Davis1970TheRO}. Our bounds on $\varepsilon_k$, however, are tighter. 
Specifically, the first term does not depend on the number of features $p_k$, and the dependence on $p_k$ is only in the second term, which is of smaller order when the signal $\sigma_{r_k}(X_k)$ is large. Instead, we have the explicit dependence on the rank $r_k$, which highlights how smaller-dimensional signals are easier to estimate.

The third bound highlights the difficulties that arise from the non-orthogonal individual subspaces $\mathcal{N}_k$. As expected,  high alignments across these subspaces makes the separation of joint and individual structures more difficult, affecting estimation accuracy.

}

\section{{Proposed Algorithm Based on Product of Projection Matrices}}\label{sec:estimation}

 In practice, our primary goal is to use the observed spectrum of $P_{\col (\widehat{X}_1)} P_{\col (\widehat{X}_2)}$ to accurately determine the joint rank. Our new theoretical insights characterize spectrum perturbation in terms of $\varepsilon_1$ and $\varepsilon_2$, both of which are unknown in practice { and depend on the estimates $\widehat X_k$.} {Specifically, Theorem~\ref{theorem:main} tells that singular values above $1-\varepsilon_1$} correspond to the joint, while those below $\varepsilon_2$ correspond to the noise directions. Therefore, we focus on the following practical questions: (i)~{ how to estimate $\col (\widehat{X}_k)$ for the signal $X_k$} in each view $k\in \{1,2\}$? (ii)~how to estimate $\varepsilon_1$ in practice to identify the top cluster of potential joint directions? (iii)~how to bound $\varepsilon_2$ to ensure the top cluster contains only joint directions and avoid overlap from random noise alignment? {From the discussion in Section~\ref{sec:noisy}, we also have the following insights: (i) rank overestimation is preferred to underestimation since it leads to smaller $\varepsilon_1$; (ii) simultaneously, overestimation pushes $\varepsilon_2$ closer to 1 which makes it harder to distinguish the joint signal from the noise; (iii) rank-to-dimension ratio, signal-to-noise (snr) ratio, and individual component alignment all have influence on the final estimation quality -- in general, high signal rank, low snr, and high degree of alignment make the problem highly intractable.} Below, we outline answers to these practical questions and summarize our proposed method.

\textbf{(i) Estimation of $\col(X_k)$.}
\label{sec:obtaining_estimates}
To estimate the signal $X_k$, we construct $\widehat{X}_k$ using truncated singular value decomposition of observed $Y_k$  for each view $k \in \{1,2\}${, an approach which allowed us to derive probabilistic bounds in Section~\ref{sec:noisy_statistical}}. {Since overestimation is preferred over underestimation, we choose the level of truncation based on the criterion in \cite{gavish2014optimal}, which tends to overestimate the signal rank.} Let $\beta_k = n / p_k$ be the aspect ratio, and 
$y_{\text{med}}$ be the median singular value of $Y_k$. 
Then we only keep those singular vectors whose corresponding singular values exceed the threshold $(0.56\beta_k^3 - 0.95\beta_k^2 + 1.82\beta_k + 1.43) \widehat{\sigma}_k$. This threshold is designed to obtain an optimal (in mean squared error) reconstruction of the signal matrix. {Because severe overestimation is still problematic as it introduces potentially many noise directions, we show how to filter out these directions in part (iii) while also we provide a diagnostic plot that allows to confirm the choice.}

\textbf{(ii) Estimation of $\varepsilon_1$ via rotational bootstrap.} \label{subsubsec:bootstrap}{
From Theorem~\ref{theorem:main}, the number of singular values of $\widehat{M} = P_{\col(\widehat{X}_1)} P_{\col(\widehat{X}_2)}$ exceeding $1-\varepsilon_1$ provides an estimate for the joint rank, $r_{\Jcal}$. This approach, however, requires an estimate for the threshold parameter $\varepsilon_1$. To this end, we propose a novel rotational bootstrap procedure. Previous bootstrap procedures, e.g., in \citet{prothero2024data}, independently resample each view's signal subspace, $U_k$. This fails to preserve the alignment between $U_1$ and $U_2$, since in high-dimensional settings, independently sampled subspaces are likely to be nearly orthogonal, particularly when the rank-to-dimension ratio is small. As a result, independent bootstraping leads to underestimation of $\varepsilon_1$, and consequently, an unreliable estimate of the joint rank. 
}

{ In contrast, our proposed method circumvents this issue by sampling the subspaces jointly, thereby preserving their relative orientation. Specifically, let $\Sigma_{\widehat{M}}$ be a diagonal matrix containing singular values of $\widehat M = P_{\col(\widehat X_1)}P_{\col(\widehat X_2)}$. We take $[U_{1,(b)}, U_{2,(b)}]$ to be the first $\widehat{r}_1 + \widehat{r}_2 $ left singular vectors of a Gaussian matrix $N$ of size $n \times n $.} We then align the signal subspaces by updating $U_{2,(b)}$ to be $U_{1,(b)} \cos (\Sigma_{\widehat{M}}) + U_{2,(b)} \sin (\Sigma_{\widehat{M}})$. By construction, the singular values of $U_{1,(b)}^\top U_{2,(b)}$ are exactly $\cos (\Sigma_{\widehat{M}})$. For each bootstrap sample, the data replicate is generated as $Y_{k,(b)} := U_{k,(b)} \widehat \Sigma_k V^\top_{k,(b)}  + \widehat{E}_k$, where $\widehat \Sigma_k$ is a matrix of singular values from $\widehat X_k$, { $V_{k,(b)} \in \mathbb{R}^{p \times \widehat{r}_k}$ are orthonormal, drawn uniformly and independently from the Haar measure}, and $\widehat E_k$ is the adjusted estimate of the noise matrix (see Appendix~C in \citet{prothero2024data}). Here $X_{k, (b)} = U_{k,(b)} \widehat \Sigma_k V^\top_{k,(b)}$ represents the signal replicate. Using data replicate $Y_{k,(b)}$, the corresponding truncated singular value decomposition of rank $\widehat r_k$ is given by $\widehat{X}_{k,(b)} := \widehat{U}_{k,(b)} \widehat D_{(b)} \widehat V_{k,(b)}^\top$. The resulting projection matrices $P_{\col(X_k),(b)} := U_{k,(b)} U^\top_{k,(b)}$ and $P_{\col(\widehat{X}_k),(b)} :=  \widehat{U}_{k,(b)}  \widehat{U}^\top_{k,(b)}$ represent one bootstrap sample of the true and estimated signal subspaces, respectively, allowing estimation of principal angles over replications. Given $B$ bootstrap replicates, we estimate $\varepsilon_1$ as:
\begin{eqnarray}
\widehat{\varepsilon}_1 &:= \frac{1}{B}\sum_{b = 1}^B \|P_{\col(X_1),(b)}(\Delta_{1,(b)}+\Delta_{2,(b)}+\Delta_{1,(b)}\Delta_{2,(b)})P_{\col(X_2),(b)}\|_2,
\label{eqn:epsilon_estimates}
\end{eqnarray}
where $\Delta_{k,(b)} = P_{\col(X_k),(b)}-P_{\col(\widehat{X}_k),(b)}, k \in \{1,2\}$. 


\textbf{(iii) Filtering out noise directions via bound on $\varepsilon_2$.} When the estimated subspaces $\col(\widehat{X}_1)$ and $\col(\widehat{X}_2)$ contain noise directions—a consequence of low signal-to-noise ratio or rank over-specification—these directions can exhibit spurious alignment purely by chance. {Such random alignments produce a cluster of singular values bounded from above by $\varepsilon_2$, as established in Theorem~\ref{theorem:main}. A large value of $\varepsilon_2$ complicates the separation of joint and noise subspaces, as it implies a potential overlap in their corresponding singular value distributions.}

To filter out noise directions, one could adapt the bootstrap methodology to estimate $\varepsilon_2$ and use it as a filtering threshold. However, we find a simpler and more direct approach to be effective. By leveraging results from random matrix theory, we found a deterministic upper bound for $\varepsilon_2$ { that depends only on the rank-to-dimension ratios. }

To motivate the analytical bound, consider a simplified scenario with no individual signals,  $P_{\col(X_K)}=P_{\Jcal}$, with $P_{\col(\widehat X_k)} = P_{\Jcal} + P_{\text{noise}_k}$, that is the estimated signal is exactly the joint signal plus orthogonal noise directions. {Formally, $P_{\text{noise}_k}$ is the projection onto the subspace $\Jcal^\perp \cap \col(\widehat X_k)$.} Then $
\varepsilon_2 = \left\|P_{{\col (X_1)}}\Delta_2 + \Delta_1P_{{\col (X_2)}}+\Delta_1\Delta_2\right\|_2 = \left\|P_{\text{noise}_1}P_{\text{noise}_2}\right\|_2$, representing the maximal singular value of the product of random projections. While $\varepsilon_2$ is generally more complex, the key intuition remains: only singular values above the threshold expected from random noise alignment can be confidently identified as the signal. 

We explicitly characterize the distribution of the entire spectrum of the product of random projections $P_{\text{noise}_1}P_{\text{noise}_2}$ by leveraging results from random matrix theory, which leads to the proposed analytical bound on $\left\|P_{\text{noise}_1}P_{\text{noise}_2}\right\|_2$. Let $P_{\text{noise}_1}$ and $P_{\text{noise}_2}$ denote two independent projection matrices onto random subspaces (drawn from a Haar measure) of $\R^n$ with ranks $r_1$ and $r_2$. Let  $q_k := r_k/n$ for $k\in \{1,2\}$ be the aspect ratio. Then the asymptotic distribution of $\lambda$, the squares of the singular values of $P_{\text{noise}_1}P_{\text{noise}_2}$, as $n\to \infty$ \citep{bouchaud2007large} has a continuous part 
\begin{equation}
f(\lambda) = \frac{\sqrt{(\lambda_+ - \lambda)(\lambda - \lambda_-)}}{2\pi\lambda(1-\lambda)}, \quad 
\lambda_{\pm} = q_1 + q_2 - 2q_1q_2 \pm 2\sqrt{q_1q_2(1-q_1)(1-q_2)},
\label{eqn:noise_bound}
\end{equation}
and two delta peaks $A_0\delta(\lambda)$ and $A_1\delta(\lambda-1)$ with $A_0 =  1 - q_1 \wedge q_2$ and $A_1 = (q_1 + q_2 -1) \vee 0$. We illustrate this distribution in Appendix~\ref{sec:noise-plot-appendix} as a function of $q_1$, $q_2$. As the number of noisy directions increases, so does the maximal singular value $\left\|P_{\text{noise}_1}P_{\text{noise}_2}\right\|_2$, reaching the value of one at $q_1=q_2=1/2$. When $q_1, q_2 < 1/2$, the maximal singular value is given by $\lambda_+^{1/2}$ in~\eqref{eqn:noise_bound}. 

{
This analytical result provides a practical filtering strategy. We propose to classify only those singular values as joint signal whose magnitude exceeds the spectral edge $\lambda_+$ (in addition to exceeding $1-\widehat \varepsilon_1$). In practice, we calculate the aspect ratios $q_k$ using the full estimated marginal ranks, $\rank(\widehat{X}_k)$. This substitution makes the bound conservative, as it provides a worst-case noise estimate by assuming all non-joint directions could be noise.

}

\textbf{Full algorithm description and diagnostics.} In summary, given the observed data matrices $(Y_1,Y_2)$, we obtain estimates of $\col (\widehat{X}_1)$ and $\col (\widehat{X}_2)$ as in (i) to construct $\widehat{M} = P_{\col (\widehat{X}_1)} P_{\col (\widehat{X}_2)}$. We estimate the rank of the joint subspace $\widehat{r}_{\mathcal{J}}$ as the number of singular values in the spectrum of $\widehat{M}$ that are above $1-\widehat{\varepsilon}_1 \vee \lambda_{+}^{1/2}$ where $\widehat{\varepsilon}_1$ and $\lambda_+$ are given by \eqref{eqn:epsilon_estimates} and \eqref{eqn:noise_bound}, respectively. We estimate the joint and individual subspaces according to \eqref{eqn:estimated_subspaces}. We summarize the steps in Appendix~\ref{sec:algo-appendix}.

We accompany our algorithm with a diagnostic plot (e.g., Figure~\ref{fig:coad} for COAD dataset), which illustrates the spectrum of the product $\widehat{M}$ together with the bounds $1-\widehat{\varepsilon}_1$ and $\lambda_+^{1/2}$. 
 The best case scenario is observing three clear clusters in the spectrum (see e.g., Figure~\ref{fig:sing-val-clustering} case of Angle = 50, SNR = 22), which provides distinct separation of joint, non-orthogonal individual, and noise irrespective of threshold estimation. More commonly, we expect to only see two clusters (noise and individual/joint). If the cluster of top values is strictly above the threshold, we can be confident those are joint. In case the cluster is in the middle of the threshold, it likely contains a combination of joint and individual. A lack of clustering structure would imply a large magnitude of noise, preventing accurate separation. In the latter case, driven by our understanding of the magnitude of $\left\|P_{\text{noise}_1}P_{\text{noise}_2}\right\|_2$ from Appendix~\ref{sec:noise-plot-appendix}, we suggest further decreasing the initial marginal rank estimates to reduce aspect ratios $q_1$, $q_2$, which subsequently reduce $\lambda_+$. In \S~\ref{sec:data}, we illustrate these ideas when analyzing real data. 

\section{Simulations on Synthetic Data}
\label{sec:simulation}

We use simulated data to evaluate the performance of our proposed product of projections decomposition method for joint and individual structure identification. For comparison, we consider methods of \citet{lock2013joint}, \citet{feng2018angle},  \citet{gaynanova2019structural}, \citet{shu2019d} and \citet{park2020integrative}. 

We simulate data from the model \eqref{eq:additive} with $Y_k = J_k + I_k + Z_k$, $k=1, 2$, and $X_k = J_k + I_k$, all representing $n \times p_k$ matrices. We let $n = 50$, $p_1 = 80$, $p_2 = 100$, $\rank(J_k) = 4$, $\rank(I_1) = 5$, and $\rank(I_2) = 4$. We sample the singular values of $J_k$ and $I_k$ independently and identically from the uniform distribution. We sample the row space of $J_k$ and $I_k$ uniformly at random according to the Haar measure. To obtain the column spaces of $J_k$ and $I_k$, we first generate an $n\times{n}$ matrix $N$ where each entry is sampled independently from a normal distribution and let $UDV^T$ be the corresponding singular value decomposition; we then set the column space of $J_k$  to be $\col(U_{[:,1:4]})$ and the column space of $I_1$ to be $\col(U_{[:,5:9]})$; we then generate the column space of $I_2$ to be $\col(\cos(\phi)U_{[:,5:9]} + \sin(\phi)U_{[:,10:14]})$,  where $\phi$ encodes the largest principle angle between the individual subspaces $\col(I_1)$ and $\col(I_2)$.  We sample the entries of the noise matrix $Z_k \in \mathbb{R}^{n \times p_k}$ independently and identically from the Gaussian distribution $\Ncal(0, s_k^2)$. From Theorem 5.32 of \citet{vershynin2010introduction}, the largest singular value of an $n\times p_k$ matrix whose entries are sampled from $\Ncal(0, s^2)$ is bounded by $s(\sqrt{n} +\sqrt{p_k})$. Therefore, to control the signal-to-noise ratio (SNR), we take the noise variance for each view to be $s_k = \|X_k\|_2/ (\mathrm{SNR}(\sqrt{n} + \sqrt{p_k}))$.

The theoretical results of \S~\ref{sec:theory} emphasize three main parameters affecting accurate identification of the subspaces: maximum principle angle between the individual subspaces $\phi$, noise level, and the marginal ranks of estimated subspaces. To test the magnitude of each of these effects, we consider $\phi \in \{30^\circ,90^\circ\}$ and $\mathrm{SNR} \in \{0.5,2\}$. For the rank estimation, we consider three scenarios: (a) we let each method estimate the ranks based on the default method's implementation; (b) we supply marginal ranks that are smaller than the truth (under-specified rank); (c) we supply marginal ranks that are larger than the truth (over-specified rank).  For (b) and (c), the miss-specification is achieved by sampling uniformly from $\{1,2,3\}$ and subtracting it from $\rank(X_k)$ in the under-specified setting or adding it to $\rank(X_k)$ in the over-specified setting, respectively. Since the implementation of \citet{lock2013joint}, \citet{gaynanova2019structural} and \citet{park2020integrative} does not allow specification of marginal ranks, we only implement (b) and (c) with methods of \citet{feng2018angle}, \citet{shu2019d} and our approach.

Choosing the right accuracy measure for unsupervised multi-view decomposition is an open challenge, as discussed in \citet{gaynanova2024comments}. Many approaches rely on the percent of variance explained or the marginal view reconstruction. However, the percent of variance explained is maximized by a separate singular value decomposition for each view, favoring a lack of joint structure and larger marginal ranks. On the other hand, signal reconstruction error, though unbiased, does not account for structural assumptions or misclassification of the joint directions as individual signals. We propose a new accuracy measure using the subspace version of the false and true discovery metric from \citet{taeb2020false}. For an estimated subspace $\widehat{\mathcal{S}}$ of a true subspace $\mathcal{S}$, we evaluate the true positive proportion (TPP) and false discovery proportion (FDP). These are then combined into a single metric,  the F-score:
\begin{eqnarray}
\label{eqn:fdr}
   \text{TPP} = \frac{\Tr(P_{\widehat{\mathcal{S}}}P_{\mathcal{S}})}{\dim(\mathcal{S})}, \quad  \text{FDP} = \frac{\Tr\left(\left(I - P_{\mathcal{S}}\right) P_{\widehat{\mathcal{S}}}\right)}{\dim(\widehat{\mathcal{S}})}, \quad
   \text{F-score} = 2 \cdot \frac{(1 - \text{FDP}) \cdot \text{TPP}}{1 - \text{FDP} + \text{TPP}}.
\end{eqnarray}
We compute the F-scores for joint and individual subspace estimates and take their average F-score as a final accuracy measure, with larger values indicating better performance.

\begin{table}[!t]
\centering
\caption{Mean F-scores \eqref{eqn:fdr} on simulated data, larger values correspond to more accurate estimation of joint and individual subspaces.}
\scalebox{0.8}{
\begin{tabular}{lcccccccccccc}
& \multicolumn{4}{c}{Under-specified rank} & \multicolumn{4}{c}{Estimated rank} & \multicolumn{4}{c}{Over-specified rank} \\

 &  \multicolumn{2}{c}{SNR = 2} & \multicolumn{2}{c}{SNR=0.5} & \multicolumn{2}{c}{SNR = 2} & 
 
 \multicolumn{2}{c}{SNR=0.5} & \multicolumn{2}{c}{SNR = 2} & \multicolumn{2}{c}{SNR=0.5}  \\
 & $90^{\circ}$ & $30^{\circ}$ & $90^{\circ}$ & $30^{\circ}$ & $90^{\circ}$ & $30^{\circ}$ & $90^{\circ}$ & $30^{\circ}$ & $90^{\circ}$ & $30^{\circ}$ & $90^{\circ}$ & $30^{\circ}$  \\

JIV & -- & -- & -- & -- & 8.27 & 7.88 & \textbf{4.74} & \textbf{4.33} & -- & -- & -- & --\\
AJI & \underline{8.54} & \underline{7.65} & 5.07 & \underline{4.56} & \underline{9.81} & \underline{8.55} & 3.85 & 3.67 & \underline{9.17} & \underline{7.70} & \underline{5.01} & \textbf{4.70} \\
SLI$^*$ & -- & -- & -- & -- & 9.61 & 7.84 & 1.49 & 1.57 & -- & -- & -- & -- \\
DCC & \textbf{8.61} & 7.59 & \underline{5.10} & 4.52 & 9.80 & 8.28 & 3.78 & 3.60 & 9.14 & 7.53 & 4.93 & 4.46 \\
UNI$^*$ & -- & -- & -- & -- & 6.73 & 6.05 & \underline{4.34} & \underline{4.05} & -- & -- & -- & -- \\
PPD & 8.44 & \textbf{7.86} & \textbf{5.12} & \textbf{4.61} & \textbf{9.91} & \textbf{9.75} & 3.86 & 3.67 & \textbf{9.19} & \textbf{8.01} & \textbf{5.07} & \underline{4.67} \\

\end{tabular}}
\label{table:simulations}
\end{table}

Table~\ref{table:simulations} displays mean F-score values across 50 replications for each combination of angle, signal-to-noise ratio, and rank specification. The results align with what is expected based on Theorem~\ref{theorem:main}: for each method, the performance deteriorates with larger alignment between individual subspaces (smaller angle), and smaller signal-to-noise ratio. In the large signal-to-noise ratio case, rank under-specification is worse than rank over-specification, which is also in agreement with the theory.  In the low signal-to-noise ratio case, the results for default rank estimation are worse than those with controlled rank miss-specification, which we suspect is due to significant challenges in estimating rank correctly in this setting, with default estimation likely leading to even larger rank estimation error than controlled miss-specification. In terms of comparison across methods, it is difficult to imagine a uniformly best-performing one since different approaches have different strengths under different scenarios. However, our proposed method achieves the highest accuracy in most settings, followed closely by \citet{feng2018angle}'s method. A more detailed empirical and theoretical comparison with \citet{feng2018angle} is in Appendix~\ref{sec:comparison_ajive}, where we show that our method can more effectively distinguish between joint and individual signals.

\section{Real Data Experiments}\label{sec:data}

\subsection{Colorectal cancer data}
\label{sec:colorectal}

We analyze colorectal cancer data from the Cancer Genome Atlas with two distinct data views: RNA sequencing (RNAseq) normalized counts and miRNA expression profiles \citep{guinney2015consensus}. We process the data as in \citet{zhang2022joint}, leading to $p_1 = 1572$ RNAseq and $p_2 = 375$ miRNA variables. Colorectal cancer has been classified into four consensus molecular subtypes based on gene expression data \citep{guinney2015consensus}, and we obtain the labels from the Synapse platform (Synapse ID syn2623706). After filtering for complete data, we arrive at $n = 167$ samples from two views with cancer subtype information. Our goal is to extract joint and individual signals and determine which signals are most predictive for cancer subtypes. We apply the proposed product of projections methods and compare it with other methods from \S~\ref{sec:simulation}. We also consider standard principal component analysis for each view as the baseline.

First, we determine the ranks corresponding to joint and individual structures for each method.
For the proposed approach, we estimate the marginal ranks using the truncation method of \citet{gavish2014optimal}.  The scree plot for each view is shown in Figure~\ref{fig:coad} (left), resulting in ranks of $40$ and $42$, respectively. We then apply the proposed method, generating the diagnostic plot of the observed product of the projections spectrum (Figure~\ref{fig:coad}, middle).  In the plot, the blue region represents the noise cutoff $\lambda_+^{1/2}$, and the green region marks the top cluster based on the bootstrap estimate of $\varepsilon_1$. Observe that the spectrum does not show clear clustering, with the noise directions falling well within the region identified by the bootstrap. This inconsistency is likely due to the high rank-to-dimension ratios, which inflate both bounds and the low signal-to-noise ratio, which affects the bootstrap estimate. To address the inclusion of noise directions, we re-examined the scree plot presented in Figure~\ref{fig:coad}. The initial marginal ranks appeared too high, so we identified the elbow point at rank $16$ for each view based on the scree plot.

After adjusting the ranks, we re-applied the proposed approach with the updated diagnostic plot (Figure~\ref{fig:coad}, right). The updated plot shows more distinct clustering, with a clear separation between lower and higher singular values. Both the bootstrap and noise bounds support this observation by identifying the cluster of high singular values as the joint structure. This analysis highlights the effectiveness of the proposed diagnostic plot, enabling practitioners to fine-tune parameters to obtain a more confident separation of joint directions.

\begin{figure}[!t]
    \centering
    \includegraphics[width=0.9\linewidth]{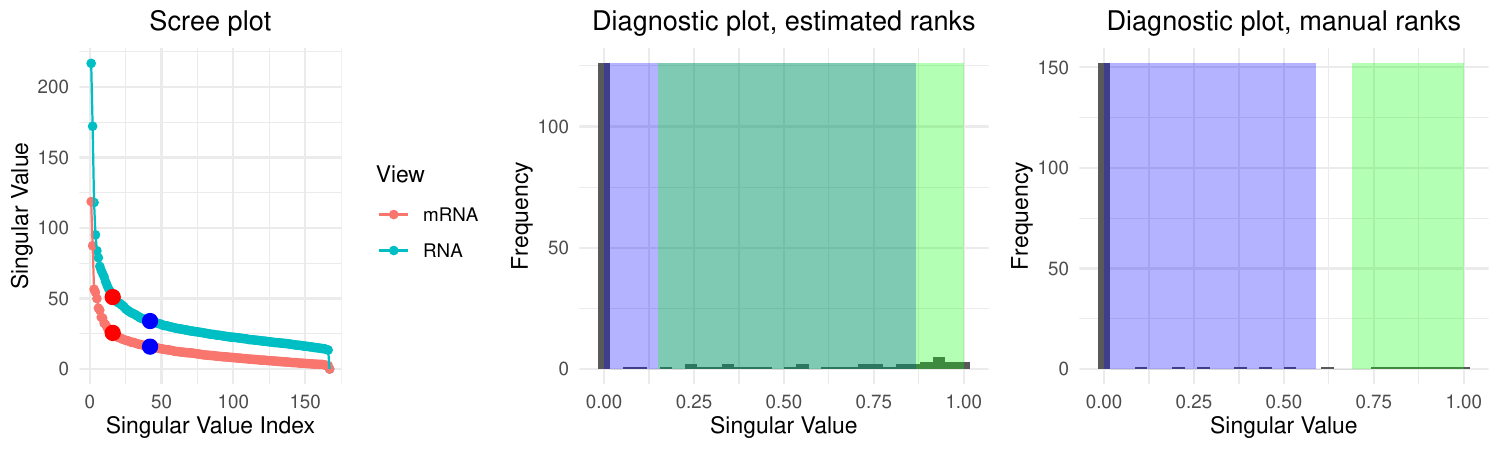}
    \caption{\textbf{Left:} Scree plot for RNASeq and mRNAseq data. \textcolor{blue}{Blue point} denotes the cutoff suggested by \citet{gavish2014optimal}. \textcolor{red}{Red point} indicates our suggested cutoff. \textbf{Middle / Right:} Spectrum of the product of projection matrices onto each view for two different marginal ranks. The \textcolor{blue}{blue region} indicates the random direction bound, $[0, \lambda_+^{1/2}]$, used to filter out the noise directions. The \textcolor{green}{green region} is given by $[1-\widehat{\varepsilon}_1,1]$ where $\widehat{\varepsilon}_1$ is a bootstrap estimate of $\varepsilon_1$.}
    \label{fig:coad}
\end{figure}

Next, we evaluate the predictive power of the estimated joint and individual components for cancer subtype classification using multinominal regression.  For a fair comparison, we use the same marginal ranks as ours for \citet{feng2018angle}, \citet{shu2019d}, and principal component analysis. For other methods, we use their default rank estimation procedures, as custom ranks could not be specified. We assess classification performance separately for the joint and the individual subspaces.  We expect that the joint components will be closely related to the cancer subtype, while the individual components would capture differences between RNA and mRNA platforms unrelated to subtypes. Table~\ref{table:coad} shows the estimated ranks, miss-classification error rates, and the smallest angle between individual subspaces for each method. As expected, most methods perform better using joint components than individual ones. The exceptions are principal components analysis and the methods of \citet{lock2013joint} and \citet{park2020integrative}. For principle components analysis and \citet{lock2013joint}, the angles between individual subspaces are small, suggesting that the joint structure may be mistaken for the individual. For \citet{park2020integrative}, the estimated marginal rank of miRNA view is much smaller compared to all other approaches, so it is possible that the joint structure was missed in that view. For joint subspace, 
\citet{gaynanova2019structural} achieves the lowest misclassification rate; however, the corresponding estimated ranks are much higher than for other approaches.  In comparison, our method offers a more parsimonious representation of the joint components while retaining strong predictive power, effectively balancing simplicity and informativeness. \citet{feng2018angle} and \citet{shu2019d} achieve slightly better miss-classification rates (by $1\%$) but at the cost of including an additional joint component. From the angles between the individual subspaces reported in Table~\ref{table:coad}, we see that this additional component corresponds to aligned individual directions with an angle of $52$ degrees between the views, which is excluded from the joint subspace in our approach. Indeed, in Figure~\ref{fig:coad}, this angle corresponds to the rotated individual directions by falling exactly in-between our noise bound, $\lambda_+^{1/2}$, and bootstrap bound, $1-\varepsilon_1$. This again highlights the value of the diagnostic plot and underscores the clear subspace separation achieved by our method.

\begin{table}[!t]
\centering
\caption{Mean cancer subtype misclassification rates in percentages together with the minimum principle angle between estimated individual subspaces.}
\scalebox{0.8}{\begin{tabular}{cccccccc}
&  \multicolumn{2}{c}{\textbf{Joint ($\Jcal$)}} & \multicolumn{2}{c}{\textbf{RNAseq} $(\Ical_1)$} & \multicolumn{2}{c}{\textbf{miRNA}  $(\Ical_2)$} & \multicolumn{1}{c}{\textbf{Angle}}\\
& Rank & Error & Rank & Error & Rank & Error  & $\angle(\Ical_1, \Ical_2)$ \\
PCA & -- & -- & 16 & 0 & 16 & 6.6 &  2 \\
JIV & 2 & 44.3 & 20 & 0 & 13 & 30.5 &  16 \\ 
AJI & 9 & 7.2 & 7 & 50.9 & 8 & 61.7  & 57 \\
SLI & 32 & 0 & 41 & 43.7 & 10 & 59.9  & 90 \\
DCC & 9 & 7.8 & 7 & 53.3 & 7 & 66.5  & 90 \\
UNI & 1 & 55.1 & 5 & 0 & 1 & 44.3  & 74 \\
PPD & 8 & 8.9 & 8 & 53.9 & 8 & 62.9  & 52 \\
\end{tabular}}
\label{table:coad}
\end{table}

\subsection{Nutrigenomic mice data}
\label{sec:nutrigenomic}

\begin{figure}[t]
    \centering
    \includegraphics[width=0.9\linewidth]{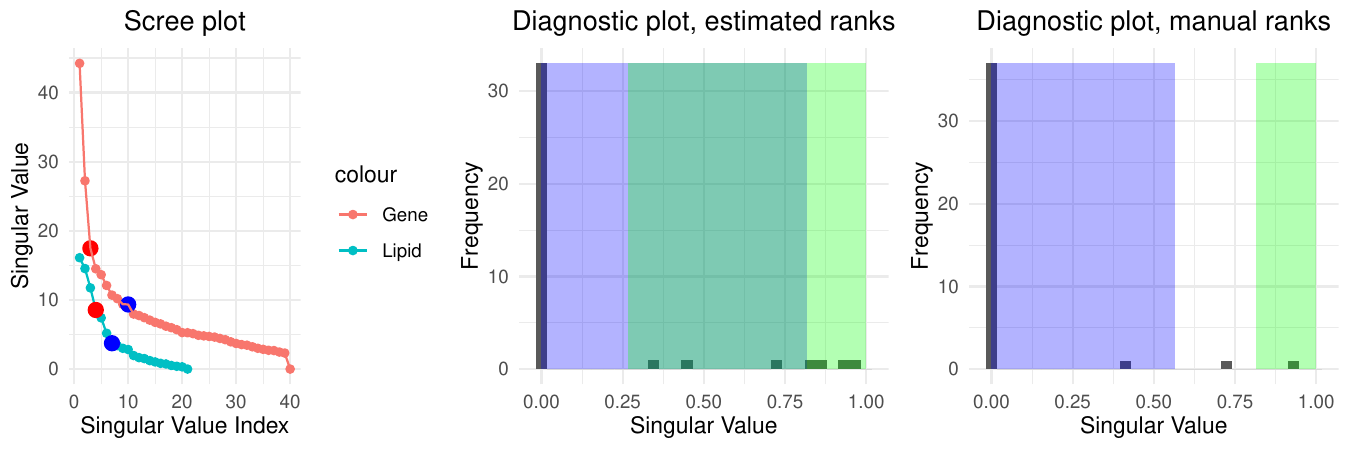}
    \caption{\textbf{Left:} Scree plot for gene expression and lipid content data. \textcolor{blue}{Blue point} denotes the cutoff suggested by \citet{gavish2014optimal}. \textcolor{red}{Red point} indicates our suggested cutoff. \textbf{Middle / Right:} Spectrum of the product of projection matrices onto each view for two different marginal ranks. The \textcolor{blue}{blue region} indicates the random direction bound, $[0, \lambda_+^{1/2}]$, used to filter out the noise directions. The \textcolor{green}{green region} is given by $[1-\widehat{\varepsilon}_1,1]$ where $\widehat{\varepsilon}_1$ is a bootstrap estimate of $\varepsilon_1$.}
    \label{fig:plot_mice}
\end{figure}

\begin{table}[t]
    \centering
\caption{Mean genotype and diet misclassification rates in percentages together with the minimum principle angle between estimated individual subspaces.}
\scalebox{0.8}{\begin{tabular}{ccccccccccc}
&  \multicolumn{3}{c}{\textbf{Joint ($\Jcal$)}} & \multicolumn{3}{c}{\textbf{Gene} $(\Ical_1)$} & \multicolumn{3}{c}{\textbf{Lipid}  $(\Ical_2)$} & \multicolumn{1}{c}{\textbf{Angles}}\\
& Rank & Type & Diet & Rank & Type & Diet & Rank & Type & Diet & $\angle(\Ical_1, \Ical_2)$ \\
PCA & -- & -- & -- & 3 & 0 & 60 & 4 & 0 & 0 & 22 \\
JIV & 2 & 0 & 27.5 & 3 & 25 & 57.5 & 3 & 30 & 12.5 & 44 \\
AJI & 2 & 0 & 55 & 1 & 52.5 & 77.5 & 2 & 47.5 & 10 & 66 \\
SLI & 6 & 0 & 0 & 8 & 45 & 42.5 & 1 & 50 & 57.5 & 90 \\
DCC & 3 & 0 & 37.5 & 1 & 45 & 65 & 1 & 45 & 47.5 & 85 \\
UNI & 2 & 0 & 62.5 & 1 & 52.5 & 77.5 & 2 & 35 & 0 & 72 \\
PPD & 1 & 0 & 80 & 2 & 37.5 & 57.5 & 3 & 30 & 0 & 44 \\
    \end{tabular}}
    \label{table:mice}
\end{table}

The nutrigenomic mice data \citep{martin2007novel} contain measurements from $n = 40$ mice from two views: gene expressions ($p_1 = 120$) and lipid content of the liver ($p_2 = 21$). The data are available as part of the multi-omics R package \citep{rohart2017mixomics}. We process the data as in \citet{yuan2022exponential}. The mice were selected from two genotypes (wild-type and PPAR-$\alpha$ mutant). They were administered five distinct diets consisting of corn and colza oils (ref), a saturated fatty acid diet of hydrogenated coconut oil (coc), an Omega6-rich diet of sunflower oil (sun), an Omega3-rich diet of linseed oil (lin), and a diet with enriched fish oils (fish). Our goal is to identify the joint and individual components across views. We expect that the joint components would be predictive of the mice genotype, while individual components for the lipid view would be predictive of the diet. 

As in \S~\ref{sec:colorectal}, we start by determining the ranks for each view.  The truncation method of \citet{gavish2014optimal} leads to ranks $11$ and $7$ for gene expressions and lipid concentrations, respectively. Figure~\ref{fig:plot_mice}, left, shows the corresponding scree plot, which suggests that the selected ranks may be too large. Our proposed diagnostic plot based on these ranks is displayed in Figure~\ref{fig:plot_mice}, middle. The noise bound is quite large, and overlaps with the bootstrap bound, suggesting that the rank reduction would be beneficial in filtering the noise directions. Previous analysis of these data has used ranks $3$ for genes and $4$ for lipids \citep{yuan2022exponential}. These ranks appear to better align with an elbow on the scree plot and lead to a diagnostic plot with a clear clustering structure and lack of overlap (Figure~\ref{fig:plot_mice}, right). We thus use these ranks for subsequent analysis.

We evaluate the predictive power of the estimated joint and individual components separately for the classification of genotype and diet based on multinomial linear regression. Table~\ref{table:mice} shows that the identified joint components for all methods perfectly separate the genotype, with our method providing the most parsimonious representation of rank equal to one. Predicting the diet from the individual lipid features turns out to be a harder task, with only our method and that of \citet{park2020integrative} achieving the perfect classification. While principle components analysis also achieves perfect diet classification, the angle between individual components of 22 degrees is small and suggests a missing joint structure. A closer look at Figure~\ref{fig:plot_mice} shows that the perfect classification of diet by our method is based on the rotated individual direction, which is not part of our joint estimate. In contrast, \citet{feng2018angle}'s method treats this direction as a joint component, leading to worse diet classification accuracy. As in the example from Section~\ref{sec:colorectal}, the diagnostic plot proves valuable in evaluating the potential over-specification of ranks and providing confidence in the separation of joint and individual components.

\section{Conclusion}

In this work, we describe a multi-view data decomposition method that naturally derives from the identifiability conditions for joint and individual subspaces. The advantages of our approach include subspace estimation guarantees, scalability, and interpretable diagnostics. 

Our work opens several avenues for future research. First, it would be valuable to rigorously extend our method to the multi-view setting with partially shared joint signals. Our current extension in Appendix~\ref{sec:multiview} assumes fully shared joint components and selects the joint rank via pairwise testing. Second, further theoretical developments could elucidate the properties of our rotational bootstrap approach and refine the asymptotic characterization of the noise bound to account for non-random subspaces. Finally, since many practical applications emphasize predictive performance, extending our method to the semi-supervised setting, as in \citet{park2021sparse, palzer2022sjive}, would be beneficial.


\section{Acknowledgements}
Armeen Taeb is supported by NSF DMS-2413074 and by the Royalty Research Fund at the University of Washington. Irina Gaynanova's research was supported by NSF DMS-2422478.

\bibliographystyle{biometrika}
\bibliography{bibliography.bib,IrinaReferences}

\newpage
\appendix
\section{Techinical details}
\subsection{Proof of Theorem \ref{theorem:main}}
\label{sec:proofs}
Our proof relies on the following lemma, which we establish first.
\begin{lemma}\label{lemma:singularbounds} Let $X, Y\in \R^{n \times n}$. If $X^{\top}Y = 0$ or $XY^{\top} = 0$, then
\begin{enumerate}
\item $\sigma_r(X+Y)\geq \max(\sigma_r(X), \sigma_r(Y))$ for all $r\geq 1$ (assuming $\sigma_1 \geq \sigma_2\geq ....0$)
\item $\sigma_1^2(X+Y) \leq \sigma_1^2(X) + \sigma_1^2(Y)$
\end{enumerate}
\end{lemma}
\begin{proof}[of Lemma~\ref{lemma:singularbounds}]
Without loss of generality, consider $X^{\top}Y = 0$ (proof is analogous for $XY^{\top} = 0$). Consider
$$
(X+Y)^{\top}(X+Y) = X^{\top}X + Y^{\top}Y.
$$
Thus, we have the eigenvalues satisfy
$$
\lambda_r\{(X+Y)^{\top}(X+Y) \} = \lambda_r(X^{\top}X + Y^{\top}Y).
$$

Weyl's inequality states:
\begin{align*}
    \lambda_n(X^{\top}X) +  \lambda_r(Y^{\top}Y) & \leq \lambda_r(X^{\top}X + Y^{\top}Y) \leq \lambda_1(Y^{\top}Y) +  \lambda_r(X^{\top}X) \\
    \lambda_n(X^{\top}X) +  \lambda_r(Y^{\top}Y) & \leq \lambda_r(X^{\top}X + Y^{\top}Y) \leq \lambda_1(X^{\top}X) +  \lambda_r(Y^{\top}Y) 
\end{align*}
{
$$
\lambda_r(X^{\top}X + Y^{\top}Y) \geq \max\left\{ \lambda_n(X^{\top}X) +  \lambda_r(Y^{\top}Y),  \lambda_n(X^{\top}X) +  \lambda_r(Y^{\top}Y)\right\}
$$

Recall that $\sigma^2_r(A) = \lambda_r(A^{\top}A)$. Since $\lambda_n(Y^{\top}Y) \geq 0$, $\lambda_n(X^{\top}X)\geq 0$, we get (1).

 Similarly,}
$$
\lambda_r(X^{\top}X + Y^{\top}Y) \leq \min\left\{\lambda_r(X^{\top}X) + \lambda_1(Y^{\top}Y), \lambda_1(X^{\top}X) + \lambda_r(Y^{\top}Y)\right\},
$$
 since $\lambda_1(Y^{\top}Y) \geq 0$, $\lambda_1(X^{\top}X)\geq 0$, taking $r=1$ yields (2).
\end{proof}

\begin{proof}[of Theorem~\ref{theorem:main}]
First, we know from Section~\ref{sec:noiseless} that
\begin{equation}
    P_{\col (X_1)} P_{\col (X_2)} =P_{\Jcal} + P_{\Ical_1}P_{\Ical_2}=P_{\Jcal} + P_{\Ncal_1}P_{\Ncal_2}= \left [U_{\Jcal}\quad U_{\Ncal_1} H \right] \begin{bmatrix}
        I_{\dim \Jcal} & 0 \\
        0 & \Sigma 
    \end{bmatrix} \left [U_{\Jcal}\quad U_{\Ncal_2} K\right]^\top,
\end{equation}
where $\Sigma$ is diagonal corresponding to the cosines of the non-zero principle angles between $\Ical_1$ and $\Ical_2$. Thus, $\sigma_r(P_{\col(X_1)}P_{\col(X_2)}) = 1$ for $1 \leq r \leq \dim (\Jcal)$ and $\sigma_r(P_{\col(X_1)}P_{\col(X_2)}) = \sigma_{r-\dim(\Jcal) + 1}$ for $\dim(\Jcal) +1 \leq r \leq \dim(\Jcal) + \rank (P_{\Ical_1} P_{\Ical_2})$.

The upper bound for each cluster follows directly from applying Weyl's inequality as
\begin{align*}
\sigma_r(\widehat M) &= \sigma_r(P_{\col(X_1)}P_{\col(X_2)} + P_{\col(X_1)}\Delta_2 + \Delta_1P_{\col(X_2)} + \Delta_1\Delta_2) \\
&\leq \sigma_r(P_{\col(X_1)}P_{\col(X_2)}) + \epsilon_2.
\end{align*}

To get the lower bound, we rewrite $\widehat M$ as:
\begin{align*}
\widehat M &= P_{\col(X_1)}P_{\col(X_2)} + P_{\col(X_1)}\Delta_2 + \Delta_1P_{\col(X_2)} + \Delta_1\Delta_2\\
&= \underbrace{P_{\col(X_1)}(I + \Delta_1 + \Delta_2 + \Delta_1\Delta_2)P_{\col(X_2)}}_{:=A_1} + \underbrace{P_{\col(X_1)}(\Delta_2 + \Delta_1\Delta_2)(I-P_{\col(X_2)})}_{:=A_2} \\
&\quad + \underbrace{(I-P_{\col(X_1)})(\Delta_1 + \Delta_1\Delta_2)P_{\col(X_2)}}_{:=A_3} + \underbrace{(I-P_{\col(X_1)})\Delta_1\Delta_2(I-P_{\col(X_2)})}_{:=A_4} 
\end{align*}
Note that by construction we have $(A_1 + A_2)^\top (A_3 + A_4) = 0$ and $A_1 A_2^\top = 0$. Then it holds that:
\begin{align*}
\sigma_r(\widehat{M}) &\geq \sigma_r(A_1 + A_2) & \text{(Lemma~\ref{lemma:singularbounds})} \\
&\geq \sigma_r(A_1)  & \text{(Lemma~\ref{lemma:singularbounds})}\\
&\geq \sigma_r(P_{\col(X_1)}P_{\col(X_2)}) - \epsilon_1. & \text{(Weyl's inequality)}   
\end{align*}
\end{proof}

\subsection{Proof of Theorem~\ref{thm:subspace_est}
}

In the proof of Theorem~\ref{thm:subspace_est}, we use the adaptation of the Davis-Kahan theorem as used in \citet{yu2015useful} and \citet{vershyninHighDimensionalProbabilityIntroduction2025}. The theorem is stated in terms of the $\sin\Theta$ distance between the subspaces. To apply the theorem, we state the definition of the $\sin \Theta$ distance and principle angles here. Specifically, the $\sin \Theta$ distance is used to measure the difference between two $n \times r$ matrices $V$ and $\widehat{V}$  with orthogonal columns. Suppose the singular values of $V^\top \widehat{V}$ are $\sigma_1 \geq \sigma_2 \geq \cdots \geq \sigma_r \geq 0$. Then we call
\[
\Theta(V, \widehat{V}) = \operatorname{diag}(\cos^{-1}(\sigma_1), \cos^{-1}(\sigma_2), \dots, \cos^{-1}(\sigma_r))
\]
as the principle angles. A quantitative measure of distance between the column spaces of $V$  $\widehat{V}$ is then $\|\sin \Theta(V, \widehat{V})\|_2$ or $\|\sin \Theta(V, \widehat{V})\|_F$, where the $\sin$ function is defined entrywise on the matrix $\Theta(V, \widehat{V})$. Note that the difference of projection matrices, $P_{\col(V)} - P_{\col(\widehat{V})}$, is related to the $\sin\Theta$ via the equality $\| P_{\col({V})} - P_{\col(\widehat{V})}\|_2$ = $\|\sin \Theta(V , \widehat{V})\|_2$. {We use the following adaptation of the Davis-Kahan theorem from \citet{vershyninHighDimensionalProbabilityIntroduction2025}.}

{
\begin{lemma}[Exercise 4.13 from \citet{vershyninHighDimensionalProbabilityIntroduction2025}]\label{lemma_kahan}Let $\Sigma,\widehat{\Sigma} \in \mathbb{R}^{n \times n}$ be symmetric with eigenvalues $\lambda_1 \geq \dots \geq \lambda_{n}$ and $\widehat{\lambda}_1 \geq \dots \geq \widehat{\lambda}_n$ respectively. Fix $1 \leq s <n$ with $\lambda_{s} - \lambda_{s+1}>0$. Let $V = (v_1,v_{2},\dots,v_s)\in\mathbb{R}^{n \times s}$ and let $\widehat{V} = (\widehat{v}_1,\widehat{v}_{2},\dots,\widehat{v}_{s}) \in \mathbb{R}^{n \times s}$ consist of the first $s$ eigenvectors of $\Sigma$ and $\widehat{\Sigma}$. Then,
$$ \|\sin \Theta(V, \widehat{V})\|_2 \leq \frac{2\|\Sigma - \widehat{\Sigma}\|_2}{\lambda_{s}-\lambda_{s+1}}.$$
\end{lemma}
}

\begin{proof}[of Theorem~\ref{thm:subspace_est}] We treat the joint and individual components separately. 

    \noindent
    \textbf{Joint components.}
    To apply Lemma~\ref{lemma_kahan}, we show that there exits an eigengap in the population-level model $S = \tfrac{1}{2}\left(P_{\col(X_1)} P_{\col(X_2)} + P_{\col(X_2)}P_{\col(X_1)}  \right)$. The matrix $\widehat{S} = \tfrac{1}{2}\left(P_{{\col}(\widehat{X}_1)} P_{{\col}(\widehat{X}_2)} + P_{{\col}(\widehat{X}_2)}P_{{\col}(\widehat{X}_1)}  \right)$ represents the finite sample estimate. 
    
    From Section~\ref{sec:noiseless}, we know that $S = P_{\Jcal} + \frac{1}{2}\left(P_{\Ncal_1}P_{\Ncal_2}+P_{\Ncal_2}P_{\Ncal_1}\right)$. Fix an orthonormal basis, $U_{\Jcal}$, for $P_{\Jcal}$. Since $\Jcal \perp \Ncal_k$, we have that $SU_{\Jcal} = U_{\Jcal}$. Hence, $S$ has at least $\dim (\Jcal)$ eigenvalues that are equal to $1$ with corresponding eigenvectors forming $\Jcal$. Now fix any orthonormal vector $u \in \spann (U_{\Jcal})^\perp$. Then we have:
    \begin{align*}
        \|Su\|_2 &= \left\|P_{\Jcal} u + \tfrac{1}{2}\left(P_{\Ncal_1}P_{\Ncal_2}+P_{\Ncal_2}P_{\Ncal_1}\right) u\right\|_2 \\
        &\leq {\left\|P_{\Jcal} u\right\|_2 +\frac{1}{2}\left\|P_{\Ncal_1}P_{\Ncal_2}(u)\right\|_2 +\frac{1}{2}\left\|P_{\Ncal_2}P_{\Ncal_1}(u)\right\|_2} & \text{(by triangle-inequality)} \\
        &\leq \left\|P_{\Ncal_2}P_{\Ncal_1}\right\|_2 \\
        &< 1 & \text{(since }\Ical_1 \cap \Ical_2 = \{\mathbf{0}\})
    \end{align*}
    {Here, the second inequality follows from $u \in \Jcal^\perp$, $\|u\|_2=1$ and the property that $\|Au\|_2 \leq \|A\|_2\|u\|_2=\|A^\top\|_2\|u\|$ for any matrix $A\in\mathbb{R}^{n \times n}$}. Therefore, the spectrum of $S$ becomes $\sigma_i(S) = 1$ for $1 \leq i \leq \dim(\Jcal)$ and $\sigma_i(S) \leq \left\|P_{\Ncal_1} P_{\Ncal_2}\right\|_2 < 1$ for $i > \dim (\Jcal)$. Fix an orthonormal basis, $\widehat{U}_{\Jcal}$, for $P_{\widehat{\Jcal}}$. By the upper bound on $\varepsilon_1$ and appealing to Theorem~\ref{theorem:main}, we have that $\mathrm{dim}(\widehat{\Jcal}) = \mathrm{dim}(\Jcal)$. Let $s :=  \mathrm{dim}(\Jcal)$, $\widehat{\Sigma} = \widehat{S}$, and apply Lemma~\ref{lemma_kahan} to obtain:
    \begin{equation*}
        \| \sin\Theta(U_{\Jcal}, U_{\widehat{\Jcal}}) \|_2 \leq \frac{{2} \| S -\widehat{S} \|_2}{1-\left\|P_{\Ncal_1} P_{\Ncal_2}\right\|_2}
    \end{equation*}
    Putting everything together, we have:
    \begin{align*}
        \| P_{\Jcal} - P_{\widehat{\Jcal}} \|_2 &= \| \sin\Theta(U_{\Jcal}, U_{\widehat{\Jcal}} )\|_2 \\ 
        &\leq \frac{{2}\left\| S- \widehat{S}\right\|_2}{1-\left\|P_{\Ncal_1} P_{\Ncal_2}\right\|_2} \\
        &\leq  \frac{  {}\left\|P_{\col(X_1)} \Delta_2 + \Delta_1 P_{\col(X_2)} + \Delta_1 \Delta_2 + \Delta_2 P_{\col(X_1)}  + P_{\col(X_2)} \Delta_1  +  \Delta_2\Delta_1\right\|_2}{1-\left\|P_{\Ncal_1} P_{\Ncal_2}\right\|_2} \\
        &= \frac{ {}\left\|R_{\Jcal} + R_{\Jcal}^\top\right\|_2}{1-\left\|P_{\Ncal_1} P_{\Ncal_2}\right\|_2}.
        \end{align*} 

        \noindent
        \textbf{Individual components.} Since $\Jcal \perp \Ical_1$ and $\Jcal \perp \Ical_2$, we have that $P_{\Jcal^\perp} P_{\col(X_1)} = P_{\Ical_1}$ and $P_{\Jcal^\perp} P_{\col(X_2)} = P_{\Ical_2}$. 
        Therefore, we have that $P_{\Jcal^\perp} P_{\col(X_k)}$ has exactly $r_{\Ical_k}$ eigenvalues equal to $1$ and the rest equal to $0$ so that the gap between the $r_{\Ical_k}$ and $r_{\Ical_k}+1$ eigenvalue of $P_{\Jcal^\perp} P_{\col(X_k)}$ is one. Let $U_{{\Ical}_k}$ consist of the top $r_{\Ical_k}$ singular vectors of $P_{\Jcal^\perp} P_{\col(X_1)}$: these span the subspace ${\Ical_k}$. Let ${U}_{\widehat{\Ical}_k}$ consist of the first $\mathrm{rank}(\widehat{X}_k) - \widehat{r}_{\Jcal}$ singular vectors of $P_{\widehat{\Jcal}^\perp} P_{\col(\widehat{X}_1)}$: these span the subspace $\widehat{\Ical}_k$. Since $\mathrm{rank}(\widehat{X}_k) = \mathrm{rank}(X_k)$ by assumption and $\widehat{r}_{\Jcal} ={r}_{\Jcal}$ from the first part of the theorem, we have that $U_{{\Ical}_k}$ and $U_{\widehat{\Ical}_k}$ have the same number of columns. Let $s :=  r_{\Ical_k}$, $\Sigma:= P_{\Jcal^\perp} P_{\col(X_k)}$, $\widehat{\Sigma} = P_{\widehat{\Jcal}^\perp} P_{\col(\widehat{X}_1)}$, and apply Lemma~\ref{lemma_kahan} to obtain:        
        
        \begin{align*}
            \| P_{\Ical_k} - P_{\widehat{\Ical}_k}\|_2 & =  \|\sin \Theta(U_{{\Ical_k}}, U_{\widehat{\Ical}_k}) \|_2 \\
            &\leq{2\left\|P_{\Jcal^\perp} P_{\col(X_k)} - P_{\widehat{\Jcal}^\perp} P_{\col(\widehat{X}_k)}\right\|_2} \\
            &{= 2\left\|P_{\Jcal^\perp} P_{\col(X_k)} - (P_{\Jcal^\perp} + \Delta_{\Jcal})(P_{\col (X_k)} - \Delta_k)\right\|_2}\\
            &{= 2\left\|P_{{\Jcal}^\perp} \Delta_k - \Delta_{\Jcal} P_{\col(X_k)} + \Delta_{\Jcal} \Delta_k \right\|_2} \\
            &{=2\|R_{\Ical_k}\|_2.}
        \end{align*}
\end{proof}



\subsection{{Proof of Theorem \ref{thm:probabilistic}
}}
\label{proof:probabilistic}
{The proof of this theorem relies on the following lemmas.
\begin{lemma}Let $L \in \mathbb{R}^{d_1 \times d_2}$ have entries that are independently and identically distributed as zero-mean Gaussian random variables with variance $\tau$. Then, with probability greater than $1-2\exp(-\max\{d_1,d_2\})$, we have that $\|L\|_2 \leq C\tau(d_1^{1/2}+d_2^{1/2})$, where $C$ is a constant (does not depend on $d_1$, $d_2$, or $\tau$).
\label{lemma:gaussian_concentration}
\end{lemma}
Lemma~\ref{lemma:gaussian_concentration} follows from Theorem 4.4.3 of \cite{vershyninHighDimensionalProbabilityIntroduction2025}.
\begin{lemma} Let $L \in \mathbb{R}^{d_1 \times d_2}$ have entries that are independently and identically distributed as zero-mean Gaussian random variables with variance $\tau$. Let $U \in \mathbb{R}^{d_1 \times t}$ and $V \in \mathbb{R}^{d_2 \times t}$ be matrices with orthonormal columns with $t \leq \min\{d_1,d_2\}$. Then $LV \in \mathbb{R}^{d_1 \times t}$ and $L^\top{U} \in \mathbb{R}^{d_2 \times t}$ are matrices with entries that are independent and identically distributed as zero-mean Gaussian random variables with variance $\tau$. 
\label{lemma:linear_algebra}
\end{lemma}
Lemma~\ref{lemma:linear_algebra} is a standard result based on the fact that linear combinations of Gaussian random variables are Gaussian. We state it as a separate lemma for convenience.}

{
\begin{lemma}Suppose Assumptions~\ref{a:noise}--\ref{a:signal_noise_ratio} hold. If $\widehat r_k = r_k$, then
we have for some absolute constant $C>0$, with probability greater than $1-\mathcal{O}(\exp\{-\max(p_k,n)\})$
$$
\|\Delta_k\|_2 \leq C\left(\delta_k\frac{n^{1/2}+r_k^{1/2}}{\sigma_{r_k}(X_k)} + \delta_k^2\frac{(n^{1/2} + p_k^{1/2})(p_k^{1/2}+r_k^{1/2})}{\sigma_{r_k}^2(X_k)}\right).
$$
If $\widehat r_k \geq r_k$, then we have for some absolute constant $C>0$, with probability greater than $1-\mathcal{O}(\exp\{-\max(p_k,n)\})$
$$
\|P_{\col(X_k)}\Delta_k\|_2 \leq C\left(\delta_k\frac{n^{1/2}+r_k^{1/2}}{\sigma_{r_k}(X_k)} + \delta_k^2\frac{(n^{1/2} + p_k^{1/2})(p_k^{1/2}+r_k^{1/2})}{\sigma_{r_k}^2(X_k)}\right).
$$
\label{lemma:probab}
\end{lemma}}

\begin{proof}[of Lemma~\ref{lemma:probab}] {We consider the two cases separately.

\vspace{0.1in}
\noindent\textbf{Correct rank estimation case with $\widehat r_k = r_k$} For simplicity, we drop subscript $k$ in the proof. Let $X$ of rank-$r$ have singular value decomposition
$$
X = [U\ U_{\perp}]\begin{pmatrix} \Sigma_1 & 0\\0&0 \end{pmatrix} [V\ V_{\perp}]^{\top}.
$$}
{Let $Y = X + Z$, and similarly partition $Y$ as follows:
$$
Y = [\widehat U\ \widehat U_{\perp}]\begin{pmatrix} \widehat \Sigma_1 & 0\\0&\widehat \Sigma_2 \end{pmatrix} [\widehat V\ \widehat V_{\perp}]^{\top},
$$
where $\widehat U \in \mathbb{R}^{n \times r}, \widehat V \in \mathbb{R}^{p \times r}$, and $\widehat X = \widehat U \widehat \Sigma_1 \widehat V^{\top}$. Then $\|\Delta\|_2 = \|P_{{\col (X)}} - P_{\col (\widehat{X})}\|_2 = \|UU^{\top} - \widehat U \widehat U^\top\|_2 = \|\widehat U_{\perp}^{\top}U\|_2$, where the last equality is true because of equality in ranks, e.g. Lemma 2.5 in \citet{chenSpectralMethodsData2021}. Consider
\begin{align*}
\widehat U_{\perp}^{\top}U &= \widehat U_{\perp}^{\top}U\Sigma_1V^{\top}V\Sigma_1^{-1} = \widehat U_{\perp}^{\top}(X)V\Sigma_1^{-1} =\widehat U_{\perp}^{\top}(Y - Z) V\Sigma_1^{-1} \\
&=\widehat U_{\perp}^{\top}(\widehat U\widehat \Sigma_1\widehat V^{\top} + \widehat U_{\perp}\widehat \Sigma_2\widehat V_{\perp}^{\top} - Z) V\Sigma_1^{-1} \\
&= \widehat \Sigma_2 \widehat V_{\perp}^{\top}V\Sigma_1^{-1} - \widehat U_{\perp}^{\top}ZV\Sigma_1^{-1}.
\end{align*}
Similarly, we can show
\begin{align*}
\widehat V_{\perp}^{\top}V = \widehat \Sigma_2 \widehat U_{\perp}^{\top}U\Sigma_1^{-1} - \widehat V_{\perp}^{\top}Z^{\top}U\Sigma_{1}^{-1}.
\end{align*}
Combining the above two displays gives
\begin{align*}
\widehat U_{\perp}^{\top}U &= \widehat \Sigma_2 ^2 \widehat U_{\perp}^{\top}U\Sigma_1^{-2} -\widehat \Sigma_2 \widehat V_{\perp}^{\top}Z^{\top}U\Sigma_{1}^{-2} - \widehat U_{\perp}^{\top}ZV\Sigma_1^{-1}.
\end{align*}
Therefore, 
\begin{align*}
\|\widehat U_{\perp}^{\top}U\|_2 \leq \|\widehat U_{\perp}^{\top}U\|_2 \|\widehat \Sigma_2\|_2^2\frac1{\sigma_r(X)^2} + \|\widehat \Sigma_2\|_2\frac1{\sigma_r(X)^2}\|\widehat V_{\perp}^{\top}Z^{\top}U\|_2 + \frac1{\sigma_r(X)}\|\widehat U_{\perp}^{\top}ZV\|_2.
\end{align*}
From Weyl's inequality and using that $X$ is rank-$r$,
$$
\|\widehat \Sigma_2\|_2\leq \|\Sigma_2\|_2 + \|Z\|_2 = \|Z\|_2.
$$
Therefore,
$$
\|\widehat U_{\perp}^{\top}U\|_2 \leq \|\widehat U_{\perp}^{\top}U\|_2 \frac{\|Z\|_2^2}{\sigma_r(X)^2} + \frac{\|Z\|_2}{\sigma_r(X)}\frac{\|\widehat V_{\perp}^{\top}Z^{\top}U\|_2}{\sigma_r(X)} + \frac{\|\widehat U_{\perp}^{\top}ZV\|_2}{\sigma_r(X)}
$$
If $\|Z\|_2\leq \frac{1}{2^{1/2}}\sigma_r(X)$, then
$$
\|\widehat U_{\perp}^{\top}U\|_2 \leq 2\left\{ \frac{\|\widehat U_{\perp}^{\top}ZV\|_2}{\sigma_r(X)} + \frac{\|Z\|_2}{\sigma_r(X)}\frac{\|\widehat V_{\perp}^{\top}Z^{\top}U\|_2}{\sigma_r(X)}\right\}.
$$
Appealing to Lemma~\ref{lemma:gaussian_concentration}, we have that $\|Z\|_2 \leq C_1\delta(n^{1/2}+p^{1/2})$ with probability greater than $1-2\exp(-\max\{n,p\})$ for some constant $C_1$. Thus, under the assumption that $\sigma_r(X)\geq \sqrt{2^{1/2}}C_1\delta(n^{1/2} + p^{1/2})$, we get that $\|Z\|_2\leq \frac1{2^{1/2}}\sigma_r(X)$ holds with high probability. Furthermore, we have with probability greater than $1-2\exp(-\max\{n,p\})$
\begin{eqnarray*}
\begin{aligned}
\|\widehat U_{\perp}^{\top}ZV\|_2 &\leq \|ZV\|_2 \leq C_2\delta(n^{1/2}+r^{1/2}),\\
\|\widehat V_{\perp}^{\top}Z^\top{U}\|_2 &\leq \|Z^\top{U}\|_2 \leq C_3\delta(p^{1/2}+r^{1/2}),
\end{aligned}
\end{eqnarray*}
where $C_2$ and $C_3$ are constants that do not depend on $n,p,r$ and $\delta$. The first inequality in each of the equations above is due to the sub-multiplicative property of the spectral norm, and the second inequality follows by combining Lemma~\ref{lemma:linear_algebra} with Lemma~\ref{lemma:gaussian_concentration}. Putting things together and bringing back the notation for every view, we have that with probability greater than $1-\mathcal{O}(-\max\{n,p_k\})$, 
\begin{eqnarray}
\left\|P_{{\col (X_k)}} - P_{\col (\widehat X_k)}\right\|_2 \lesssim \frac{\delta_k(n^{1/2}+r_k^{1/2})}{\sigma_{r_k}(X_k)}+ \frac{\delta_k^2(n^{1/2} + p_k^{1/2})(p_k^{1/2}+r_k^{1/2})}{\sigma_{r_k}(X_k)^2}.
\label{eqn:equaiton_inter}
\end{eqnarray}

\vspace{0.1in}
\noindent \textbf{Rank overestimation case where $\widehat r_k \geq r_k$}: As in the previous case, we drop subscript $k$ in the proof for notational convenience. Observe that
$$
\|P_{{\col (X)}}\Delta\|_2 = \|UU^{\top} -  UU^{\top}\widehat U \widehat U^{\top}\|_2 = \|UU^{\top}\widehat U_{\perp}\widehat U_{\perp}^{\top}\| = \|U^{\top}\widehat U_{\perp}^{\top}\|_2.
$$
Here, $\widehat U \in \mathbb{R}^{n \times \widehat r}$ is the first $\widehat r$ left singular vectors of $Y$ and $\widehat U_\perp \in \mathbb{R}^{n \times (\mathrm{rank}(Y) -\widehat r)}$ represents the remaining left singular vectors corresponding to the non-zero singular values. Consider as before the singular value decomposition
$$
X = [U\ U_{\perp}]\begin{pmatrix} \Sigma_1 & 0\\0&0 \end{pmatrix} [V\ V_{\perp}]^{\top}.
$$
Let $Y = X + Z$, and similarly partition $Y$ based on singular value decomposition as follows
$$
Y = [\widetilde U\ \widetilde U_{\perp}]\begin{pmatrix} \widetilde \Sigma_1 & 0\\0&\widetilde \Sigma_2 \end{pmatrix} [\widetilde V\ \widetilde V_{\perp}]^{\top},
$$
where $\widetilde U \in \mathbb{R}^{n \times r}$ and $\widetilde V \in \mathbb{R}^{p \times r}$ consists of the first $r$ columns of $\widehat U$ and $\widehat V$, respectively. Further, $\widetilde U_\perp$ and $\widetilde V_\perp$ consist of the remaining columns of $\widehat U$ and $\widehat V$, respectively. Consider
\begin{align*}
\widehat U_{\perp}^{\top}U &= \widehat U_{\perp}^{\top}U\Sigma_1V^{\top}V\Sigma_1^{-1} = \widehat U_{\perp}^{\top}(X)V\Sigma_1^{-1} =\widehat U_{\perp}^{\top}(Y - Z) V\Sigma_1^{-1} \\
&=\widehat U_{\perp}^{\top}(\widetilde U\widetilde\Sigma_1\widetilde V^{\top} + \widetilde U_{\perp}\widetilde \Sigma_2\widetilde{V}_{\perp}^{\top} - Z) V\Sigma_1^{-1} \\
&= \widehat U_{\perp}^{\top} \widetilde U_\perp \widetilde \Sigma_2 \widehat V_{\perp}^{\top}V\Sigma_1^{-1} - \widehat U_{\perp}^{\top}ZV\Sigma_1^{-1}.
\end{align*}
Here, we have used the fact that $\widehat U_\perp ^\top \widetilde U = 0$. Similarly, we can show
\begin{align*}
\widehat V_{\perp}^{\top}V = \widehat V_{\perp}^{\top} \widetilde V \widetilde \Sigma_2 \widehat U_{\perp}^{\top}U\Sigma_1^{-1} - \widehat V_{\perp}^{\top}Z^{\top}U\Sigma_{1}^{-1}.
\end{align*}
Combining the above two displays gives
\begin{align*}
\widehat U_{\perp}^{\top}U &= \widehat U_{\perp}^{\top} \widetilde U_\perp \widetilde \Sigma_2\widehat V_{\perp}^{\top}V \widetilde\Sigma_2  \widehat U_{\perp}^{\top}U\Sigma_1^{-2} -\widehat U_{\perp}^{\top} \widetilde U_\perp \widetilde \Sigma_2 \widehat V_{\perp}^{\top}Z^{\top}U\Sigma_{1}^{-2} - \widehat U_{\perp}^{\top}ZV\Sigma_1^{-1}.
\end{align*}
Therefore, using triangle inequality, the sub-multipicative property of spectral norms, that $\|\widehat U_{\perp}^{\top} \widetilde U_\perp\|_2 \leq 1$, and that $\|\Sigma_1^{-2}\|_2 =\frac{1}{\sigma_r(X)^2}$, we obtain:
\begin{align*}
\|\widehat U_{\perp}^{\top}U\|_2 \leq \|\widehat U_{\perp}^{\top}U\|_2 \|\widetilde\Sigma_2\|_2^2\frac1{\sigma_r(X)^2} + \|\widetilde\Sigma_2\|_2\frac1{\sigma_r(X)^2}\|\widehat V_{\perp}^{\top}Z^{\top}U\|_2 + \frac1{\sigma_r(X)}\|\widehat U_{\perp}^{\top}ZV\|_2.
\end{align*}
As in the previous case, appealing to Weyl's inequality and using that $X$ is rank-$r$, we have that $\|\widehat \Sigma_2\|_2 \leq \|Z\|_2$. Following the same logic as the correct rank case, we conclude that:
$$
\|P_{{\col (X_k)}}\Delta_k\|_2 = \|U^{\top}\widehat U_{\perp}^{\top}\|_2 \lesssim \frac{\delta_k(n^{1/2}+r_k^{1/2})}{\sigma_{r_k}(X_k)}+ \frac{\delta_k^2(n^{1/2} + p_k^{1/2})(p_k^{1/2}+r_k^{1/2})}{\sigma_{r_k}(X_k)^2}.
$$}
\end{proof}

\begin{proof}[~of Theorem~\ref{thm:probabilistic}] {We consider the bounding of $\varepsilon_1$, $\varepsilon_2$, and the estimation accuracy of subspaces separately.  

\vspace{0.1in}
\noindent \textbf{Bounding $\varepsilon_1$:} Recall,
\begin{eqnarray*}
\begin{aligned}
    \varepsilon_1 &:= \left\|P_{{\col (X_1)}}\left(\Delta_1+\Delta_2+\Delta_1\Delta_2\right)P_{{\col (X_2)}}\right\|_2,
\end{aligned}
\end{eqnarray*}
with $\Delta_k := P_{{\col (X_k)}} - P_{\col (\widehat{X}_k)}$. By triangle inequality, symmetry of projection matrices, and sub-multipicative property of the spectral norm, we have that:
\begin{eqnarray*}
\varepsilon_1 \leq \left\|P_{{\col (X_1)}}\Delta_1\right\|_2+\left\|P_{{\col (X_2)}}\Delta_2\right\|+\left\|P_{{\col (X_1)}}\Delta_1\right\|_2\left\|P_{{\col (X_2)}}\Delta_2\right\|_2.
\end{eqnarray*}
Under the assumption $\widehat r_k \geq r_k$, we appeal to Lemma~\ref{lemma:probab} to conclude that with probability greater than $1-\mathcal{O}(\exp\{-\max(p_{\min},n)\})$:
$$
\varepsilon_1 \lesssim \frac{\delta_{\max}(n^{1/2}+r_{\max}^{1/2})}{\sigma_{r}(X)} + \frac{\delta_{\max}^2(n^{1/2}+p_{\max}^{1/2})(p_{\max}^{1/2}+r_{\max}^{1/2})}{\sigma_{r}(X)^2}.
$$
Here, we have used that $\|\mathcal{P}_{\col(X_1)}\Delta_1\|_2 \|\mathcal{P}_{\col(X_2)}\Delta_2\|_2 \lesssim \max_k\|\mathcal{P}_{\col(X_k)}\Delta_k\|_2$ for $k = 1,2$ due to Assumption~\ref{a:signal_noise_ratio}.

\vspace{0.1in}
\noindent \textbf{Bounding $\varepsilon_2$:} Recall,
\begin{eqnarray*}
\begin{aligned}
  \varepsilon_2 &:= \left\|P_{{\col (X_1)}}\Delta_2 + \Delta_1P_{{\col (X_2)}}+\Delta_1\Delta_2\right\|_2.
\end{aligned}
\end{eqnarray*}
Using triangle inequality and the sub-multipicative property of the spectral norm, we have that:
\begin{align*}
\varepsilon_2 \leq \|\Delta_1\|_2+ \|\Delta_2\|_2 + \|\Delta_1\|_2 \|\Delta_2\|_2.
\end{align*}
Under the correctly specified rank $\widehat r_k = r_k$, we appeal to Lemma~\ref{lemma:probab} to conclude that with probability greater than $1-\mathcal{O}(\exp\{-\max(p_{\min},n)\})$:
$$
\varepsilon_2 \lesssim \frac{\delta_{\max}(n^{1/2}+r_{\max}^{1/2})}{\sigma_{r}(X)} + \frac{\delta_{\max}^2(n^{1/2}+p_{\max}^{1/2})(p_{\max}^{1/2}+r_{\max}^{1/2})}{\sigma_{r}(X)^2}.
$$
Here, we have used the property that $\|\Delta_k\|_2^2 \lesssim \|\Delta_k\|_2$ for $k = 1,2$ due to Assumption~\ref{a:signal_noise_ratio}. 

\vspace{0.1in}
\noindent \textbf{Bounding subspace estimation errors:} Since $\widehat r_J = r_J$ and $\widehat r_k = r_k$, appealing to Theorem~\ref{thm:subspace_est}, we have that:
\begin{eqnarray*}
\begin{aligned}
\left\| P_{\Jcal} - P_{\widehat{\Jcal}} \right\|_2 &\leq  \frac{\left\|R_{\Jcal} + R^\top_{\Jcal}\right\|_2}{1-\left\|P_{\Ncal_1} P_{\Ncal_2}\right\|_2},\\
\left\|P_{\Ical_k} - P_{\widehat{\Ical}_k}\right\|_2 &\leq {2}\|R_{\Ical_k}\|_2,
\end{aligned}
\end{eqnarray*}
where ${R}_{\Jcal} := P_{\col(X_1)} \Delta_2 + \Delta_1 P_{\col(X_2)} + \Delta_1 \Delta_2$ and {$R_{\Ical_k} := P_{\Jcal^\perp} \Delta_k - \Delta_{\Jcal} P_{\col(X_k)}+ \Delta_{\Jcal} \Delta_k$} and $\Delta_{\Jcal} :=  P_{\Jcal} - P_{\widehat{\Jcal}}$. Using triangle inequality, sub-multipicative property of spectral norm, and that the spectral norm of a projection matrix is one, we have that:
\begin{eqnarray*}
\begin{aligned}
\|{R}_{\Jcal}\|_2 &\leq \|\Delta\|_2 + \|\Delta\|_1 + \|\Delta_1\|_2\|\Delta_2\|_2, \\
\|{R}_{\Ical_k}\|_2 &\leq \|\Delta_k\|_2 + \|\Delta_{\Jcal}\|_2 +  \|\Delta_k\|_2\|\Delta_{\Jcal}\|_2.
\end{aligned}
\end{eqnarray*}
Appealing to Lemma~\ref{lemma:probab}, we have with probability greater than $1-\mathcal{O}(\exp\{-\max(p_{\min},n)\})$:
\begin{eqnarray*}
\begin{aligned}
    \left\| P_{\Jcal} - P_{\widehat{\Jcal}} \right\|_2 \lesssim \frac{1}{1-\left\|P_{\Ncal_1} P_{\Ncal_2}\right\|_2}\left\{\frac{\delta_{\max}(n^{1/2}+r_{\max}^{1/2})}{\sigma_{r}(X)} + \frac{\delta_{\max}^2(n^{1/2}+p_{\max}^{1/2})(p_{\max}^{1/2}+r_{\max}^{1/2})}{\sigma_{r}(X)^2}\right\}, \\
    \left\| P_{\Ical_k} - P_{\widehat{\Ical_k}} \right\|_2 \lesssim \frac{1}{1-\left\|P_{\Ncal_1} P_{\Ncal_2}\right\|_2}\left\{\frac{\delta_{\max}(n^{1/2}+r_{\max}^{1/2})}{\sigma_{r}(X)} + \frac{\delta_{\max}^2(n^{1/2}+p_{\max}^{1/2})(p_{\max}^{1/2}+r_{\max}^{1/2})}{\sigma_{r}(X)^2}\right\}.
 \end{aligned}
\end{eqnarray*}}  
\end{proof}


\subsection{Algorithm implementation}
\label{sec:algo-appendix}
We provide the full implementation of our method in Algorithm~\ref{algo:main}. Recall from Section~\ref{sec:estimation}, given the observed data matrices \( (Y_1, Y_2) \), we first obtain estimates of \( \text{col}(\widehat{X}_1) \) and \( \text{col}(\widehat{X}_2) \) following \citet{gavish2014optimal}. Using these estimates, we construct \( \widehat{M} = P_{\text{col}(\widehat{X}_1)} P_{\text{col}(\widehat{X}_2)} \). We then estimate the rank of the joint subspace \( \widehat{r}_{\mathcal{J}} \) by counting the number of singular values in the spectrum of \( \widehat{M} \) that exceed \( (1 - \widehat{\varepsilon}_1) \vee \lambda_{+}^{1/2} \), where \( \widehat{\varepsilon}_1 \) and \( \lambda_{+} \) are defined in equations \eqref{eqn:epsilon_estimates} and \eqref{eqn:noise_bound}, respectively. The joint and individual subspaces are estimated according to equation \eqref{eqn:estimated_subspaces}.

\FloatBarrier
\begin{algorithm}[!t]
\SetAlgoLined
\KwIn{Data matrices from two views $Y_1 \in \mathbb{R}^{n \times p_1}$ and $Y_2 \in \mathbb{R}^{n \times p_2}$.}

\begin{enumerate}
    \item \textbf{Estimate marginal view signals} $\widehat{X}_k$ by truncated singular value decomposition of $Y_k$ with rank as in \citet{gavish2014optimal}
\\
    \item \textbf{Compute perturbation bound $\widehat \varepsilon_1$} by using rotational bootstrap in \eqref{eqn:epsilon_estimates}.
\\
      \item \textbf{Compute noise bound $\lambda_+$} based on~\eqref{eqn:noise_bound} with $q_k = \rank(\widehat{X}_k)/n$.
\\
    \item \textbf{Estimate joint rank}  $\widehat{r}_{\Jcal} = |\{ \sigma(\widehat{M}) > \lambda_+^{1/2} \vee (1-\widehat{\varepsilon}_1)\} |$ with
   $\widehat{M} = P_{\col(\widehat{X}_1)} P_{\col(\widehat{X}_2)}$
\\
    \item \textbf{Estimate joint subspace} $\widehat{\Jcal}$ as the span of the first $\widehat{r}_{\Jcal}$ singular vectors of $\widehat{S} =(\widehat{M}+\widehat{M}^\top)/2$
\\
    \item \textbf{Estimate individual subspaces} 
$\widehat{\Ical}_k$ as span of the first left $\rank(\widehat{X}_k)-\widehat{r}_{\Jcal}$ singular vectors of $P_{\col(\widehat{X}_k)}(I-P_{\widehat{\Jcal}})$.
\end{enumerate}
\KwOut{Joint subspace $\widehat\Jcal$ and individual subspaces $\widehat\Ical_1$ and $\widehat\Ical_2$; a histogram of the spectrum of $\widehat{M}$ overlaid with $\lambda_+^{1/2}$, $1-\widehat{\varepsilon}_1$ and the distribution $f(\lambda)$ in \eqref{eqn:noise_bound}.}
\caption{Algorithmic implementation of our method}
\label{algo:main}
\end{algorithm}
\FloatBarrier

\subsection{Extension to multi-view case}
\label{sec:multiview}
For the case of $k > 2$, we establish the following lemma on the identifiability, motivated by \citet{feng2018angle}. The main difference from the previous Lemma~\ref{lemma:identifiability} is that for the multi-view case, we further assume absence of the partially shared strctures. That is the joint subspace is a shared between all the view. 

\begin{lemma}[Identifiability] 
\label{lemma:identifiability-multiview}
Given a set of subspaces $\{\col(X_1), \dots, \col(X_K)\}$, there is a unique set of subspaces $\{\Jcal, \Ical_1, \dots, \Ical_K\}$ such that:
\begin{enumerate}
    \item $\col(X_k) = \Jcal \oplus \Ical$ with $\Jcal \subseteq \col(X_k)$
    \item $\Jcal \perp \Ical_k$
    \item $\bigcap\Ical_k = \{\mathbf{0}\}$.
\end{enumerate}
\end{lemma}

With the above in mind, we propose to identify the joint rank by testing each pair of views separately and setting the joint rank as the the minimum pairwise estimated joint rank. Specifically, we set $\dim \widehat{\Jcal} = \min_{1\leq i<j \leq k} \dim \widehat \Jcal_{ij}$, where $\widehat\Jcal_{ij}$ is the joint identified by considering views $i$ and $j$ respectively. To understand why we take the minimum, we consider the following example. Suppose $\dim\widehat{J}_{12} > \dim\widehat{J}_{13}$. The directions that have been identified for each pair are above the strict noise bound and the bootstrap bound. This means that we can rule out the possibility that some directions in $\widehat{\Jcal}_{12}$ or $\widehat{\Jcal}_{13}$ are the result of pure noise. This means the estimate $\dim\widehat{J}_{12}$ includes highly aligned signal directions between the two views. In general, our bootstrap bound always overestimates the true bound for the joint as we highlight in Section~\ref{sec:add_experiments}. Therefore, we are confident that this aligned directions in $\widehat{J}_{12}$ must correspond to the highly rotated individual subspaces, and not the corrupted joint signal. Therefore, by the above logic, we must exclude these directions from our final estimate of the joint. 

We estimate the joint subspace by taking the top $\dim \widehat{\Jcal}$ singular vectors of the following symmetric product of projections $\frac{1}{|\Pcal|}\sum_{(i_1, i_2, \dots, i_k) \in \Pcal} P_{\col(\widehat{X}_{i_1})}P_{\col(\widehat{X}_{i_2})}\dots P_{\col(\widehat{X}_{i_k})}$, where $\Pcal$ is the set of all permutations of $\{1,2,\dots, k\}$. In Theorem~\ref{thm:subspace_est_multiview}, we show that a similar subspace estimation error bound holds for the multi-view case as in the case of two views. 

\begin{theorem}[Estimation Error for $K>2$]
\label{thm:subspace_est_multiview}
Suppose our estimate of the joint rank is correct, i.e., $\dim \widehat{\Jcal} = \dim \Jcal$. Define $W_{i_j}(\Scal) = P_{\col(X_{i_j})} \Id[i_j\in \Scal] + \Delta_{i_j}\Id[i_j\not\in \Scal]$ and  {$R_{\Ical_k} := P_{\Jcal} \Delta_k - \Delta_{\Jcal} P_{\col(X_k)} + \Delta_{\Jcal} \Delta_k$, where $\Delta_{\Jcal} :=  P_{\Jcal} - P_{\widehat{\Jcal}}$} and $\Delta_k := P_{{\col (X_k)}} - P_{\col (\widehat{X}_k)}$ .
Additionally, let $\Pcal$ be the set of all permutations of $\{1,, 2, \dots, K\}$. Then, the deviation of the estimated joint subspace from the true joint subspace is bounded by:
$$\|P_{\Jcal} - P_{\widehat{\Jcal}} \|_2 \leq \frac{{2} \left\|\sum_{\Scal \subset \{1, 2, \dots, K\}} \prod_{j=1}^K W_{i_j}(\Scal) \right\|_2}{|\Pcal| \left(1-\frac{1}{|\Pcal|} \sum_{(i_1, i_2, \dots, i_k)\in \Pcal} \left\| P_{\Ical_{i_1}}P_{\Ical_{i_2}}\dots P_{\Ical_{i_k}}\right\|_2\right)}.$$
Furthermore, if the rank of each subspace view is correctly estimated, i.e. $\rank(\widehat{X}_k) = \rank(X_k))$, then, the deviation of the estimated individual subspace from the true individual subspace is bounded by: 
$$\|  P_{\Ical_k} - P_{\widehat{\Ical}_k}\|_2 \leq  {2} \|R_{\Ical_k}\|_2.$$
\end{theorem}

\begin{proof} We adapt the proof from Theorem~\ref{thm:subspace_est} to the multi-view case. Both proofs rely on the version of the Davis-Kahan theorem of \citet{yu2015useful}.

    \noindent
    \textbf{Joint components.}
    To apply Davis-Kahan theorem, we show that there exits an eigengap in the population-level model $S = \frac{1}{|\Pcal|}\sum_{(i_1, i_2, \dots, i_k) \in \Pcal} P_{\col({X}_{i_1})}P_{\col({X}_{i_2})}\dots P_{\col({X}_{i_k})}$
    Similar to Section~\ref{sec:noiseless}, we use the orthogonality assumption $\Jcal \perp \Ical_k$ to obtain $S = P_{\Jcal} + \frac{1}{|\Pcal|}\sum_{(i_1, i_2, \dots, i_k)\in \Pcal} P_{\Ical_{i_1}}P_{\Ical_{i_2}}\dots P_{\Ical_{i_k}}$. Fix an orthonormal basis, $U_{\Jcal}$, for $P_{\Jcal}$. Since $\Jcal \perp \Ical_k$, we have that $SU_{\Jcal} = U_{\Jcal}$. Hence, $S$ has at least $\dim \Jcal$ eigenvalues that are equal to $1$ with corresponding eigenvectors forming $\Jcal$. Now fix any orthonormal vector $u \in \spann U_{\Jcal}^\perp$. Then we have:
    \begin{align*}
        \|Su\|_2 &= \left\|P_{\Jcal} u + \frac{1}{|\Pcal|}\sum_{(i_1, i_2, \dots, i_k)\in \Pcal} P_{\Ical_{i_1}}P_{\Ical_{i_2}}\dots P_{\Ical_{i_k}} u\right\|_2 \\
        &\leq \frac{1}{|\Pcal|} \sum_{(i_1, i_2, \dots, i_k)\in \Pcal} \left\| P_{\Ical_{i_1}}P_{\Ical_{i_2}}\dots P_{\Ical_{i_k}}\right\|_2 & \text{(by triangle-inequality)} \\
        &< 1. & \text{(since } \bigcap \Ical_k = \{\mathbf{0}\})
    \end{align*}
    Therefore, the spectrum of $S$ becomes $\sigma_r(S) = 1$ for $1 \leq r \leq \dim\Jcal$ and $\sigma_r (S) < 1$ for $r > \dim \Jcal$. 

    From here, applying Davis-Kahan, we have that:
    \begin{equation*}
        \| \sin\Theta(U_{\Jcal}, U_{\widehat{\Jcal}}) \|_2 \leq \frac{{2}  \| S - \widehat{S}\|_2}{1-\frac{1}{|\Pcal|} \sum_{(i_1, i_2, \dots, i_k)\in \Pcal} \left\| P_{\Ical_{i_1}}P_{\Ical_{i_2}}\dots P_{\Ical_{i_k}}\right\|_2}
    \end{equation*}
    Putting everything together, we have:
    \begin{align*}
        \| P_{\Jcal} - P_{\widehat{\Jcal}} \|_2 &=  \| \sin\Theta(U_{\Jcal}, U_{\widehat{\Jcal}} ) \|_2 \\
        &\leq \frac{{2}\|  S- \widehat{S}\|_2}{1-\frac{1}{|\Pcal|} \sum_{(i_1, i_2, \dots, i_k)\in \Pcal} \left\| P_{\Ical_{i_1}}P_{\Ical_{i_2}}\dots P_{\Ical_{i_k}}\right\|_2} \\
        &\leq  \frac{{2}\left\| \sum_{\Scal \subset \{1, 2, \dots, K\}} \prod_{j=1}^K W_{i_j}(\Scal) \right\|_2}{|\Pcal| \left(1-\frac{1}{|\Pcal|} \sum_{(i_1, i_2, \dots, i_k)\in \Pcal} \left\| P_{\Ical_{i_1}}P_{\Ical_{i_2}}\dots P_{\Ical_{i_k}}\right\|_2\right)},
        \end{align*} 
        where $W_{i_j}(\Scal) = P_{\col(X_{i_j})}$ if $i_j \in \Scal$ and $W_{i_j}(\Scal) = \Delta_{i_j}$ otherwise. 

        \noindent
        \textbf{Individual components.} The case for individual components is the same as in Theorem~\ref{thm:subspace_est}.
\end{proof}

\subsection{Detailed comparison with related methods}
\label{sec:comparison_ajive}

\begin{figure}[!t]
    \centering
    \includegraphics[width=0.9\linewidth]{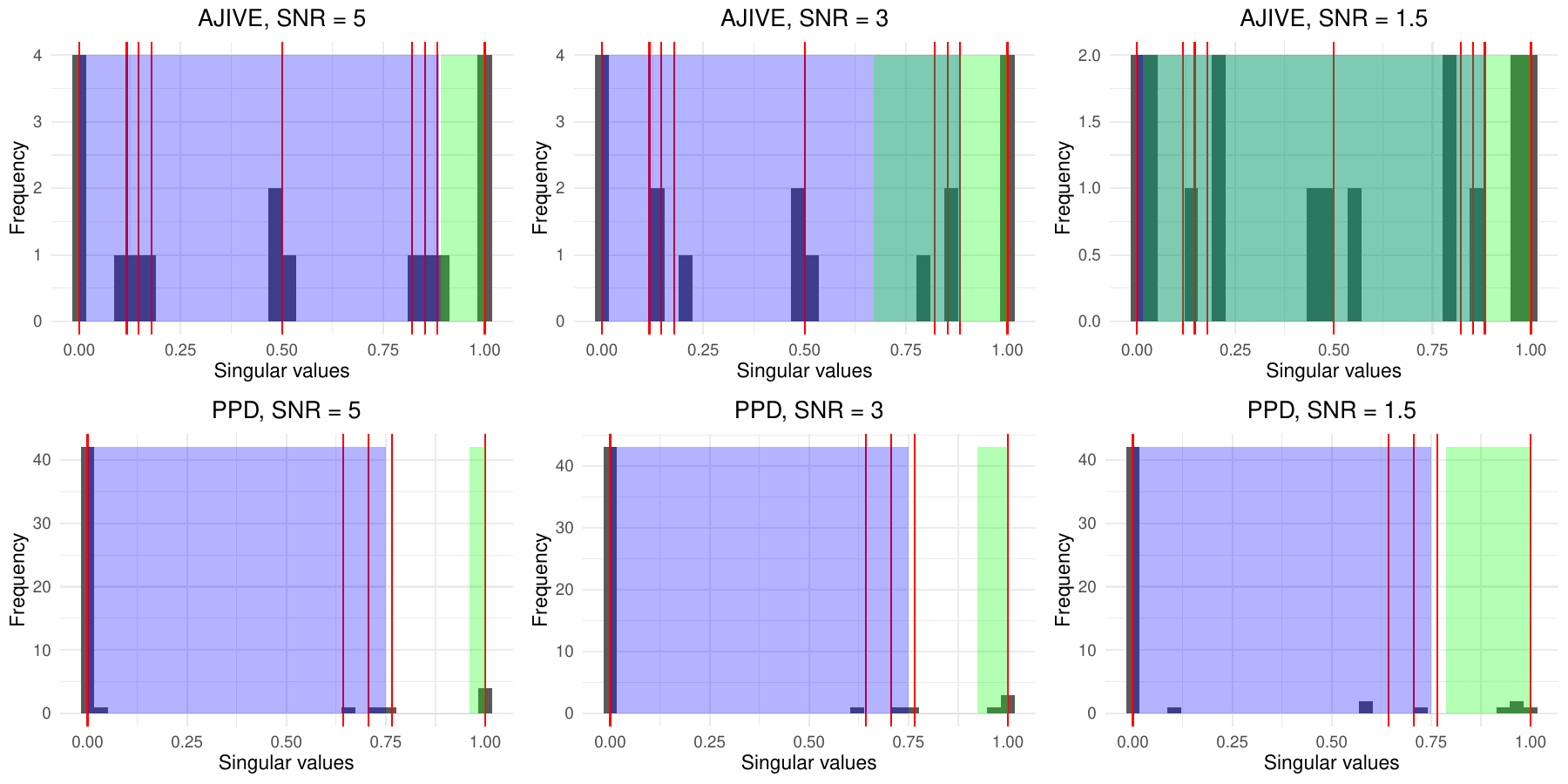}
    \caption{Simulated data of \S~\ref{sec:simulation} with angle $\phi =40^\circ$ and varying SNR. \textbf{Top row:} Spectrum of $\widehat{S}_{sum}$ analyzed by AJIVE \citep{feng2018angle}. \textcolor{blue}{Blue region} indicates the noise bound; \textcolor{green}{green region} indicates the joint Wedin bound, both computed using bootstrap. \textbf{Bottom row: } Spectrum of $\widehat{M}$ analyzed in our approach. \textcolor{blue}{Blue region} indicates the noise bound from~\eqref{eqn:noise_bound}; \textcolor{green}{green region} indicates the joint bound computed using bootstrap from~\eqref{eqn:epsilon_estimates}. For both rows, \textcolor{red}{red lines} indicate the true spectrum of $S_{sum}$ and $M$, respectively.}
    \label{fig:ppd-diagnostic-simulations}
\end{figure}

Our approach builds on the principal angle formulation by analyzing the product of projections, which directly reveals subspace alignment. This formulation was first explored for multi-view data in \citet{feng2018angle}, where the authors instead studied the sum of the projection matrices $S_{sum} = (P_{\col(X_1)} + P_{\col(X_2)})/2$ rather than the product. Applying Lemma~\ref{lemma:identifiability} to the sum gives
$$
S_{sum} = P_{\Jcal} + P_{\Ocal_1}/2 + P_{\Ocal_2}/2 + (P_{\Ncal_1} + P_{\Ncal_2})/2,
$$
which reveals that its spectrum has $\dim(\Jcal)$ eigenvalues equal to 1 with eigenvectors corresponding to joint subspace, $\dim(\Ocal_1) + \dim(\Ocal_2)$  eigenvalues equal to 1/2 with eigenvectors corresponding to orthogonal individual subspaces, and $\dim(\Ncal_1) + \dim(\Ncal_2)$ eigenvalues of the form $\{\sigma_i\}^2 = \{1 \pm \cos(\phi_i)\}/2$, where $\phi_i$ are $\dim(\Ncal_1)$ principal angles between non-orthogonal $\Ncal_1$ and $\Ncal_2$. The gap separating joint from individual subspaces is of the size $\{1-\sigma_{\max}(\Sigma)\}/2$. In contrast, $ P_{\col (X_1)} P_{\col (X_2)}$ has only $\dim(\Jcal) + \dim(\Ncal_1)$ non-zero singular values, and the gap separating joint from individual subspaces is twice larger, $\{1-\sigma_{\max}(\Sigma)\}$.

The above comparison of the population quantities directly affects the resulting estimation. Since in practice only estimates $\col(\widehat X_1)$ and $\col(\widehat X_2)$ are available, the resulting $\widehat S_{sum}$ and $P_{\col (\widehat X_1)} P_{\col (\widehat X_2)}$ both have perturbed spectrum. Since $S_{sum}$ has a smaller spectral gap between joint and individual subspaces, and a larger number of non-zero singular values unrelated to joint structure $\dim(\Ocal_1) + \dim(\Ocal_2) +  \dim(\Ncal_1) + \dim(\Ncal_2)$, $\widehat S_{sum}$ is more affected by noise perturbations than $P_{\col (\widehat X_1)} P_{\col (\widehat X_2)}$. As shown in Figure~\ref{fig:ppd-diagnostic-simulations}, the singular values from \citet{feng2018angle} do not exhibit the clustering observed in our method. 

Furthermore, estimating joint structure based on the product of projections has advantages beyond rank estimation.  To illustrate this, consider the symmetrized version
$$
\widehat{S} = (P_{\col (\widehat X_1)} P_{\col (\widehat X_2)} + P_{\col (\widehat X_2)} P_{\col (\widehat X_1)})/2.
$$
The two objects are related as $2\widehat S_{sum}^2 - \widehat S_{sum} = \widehat{S}$.
 Therefore, non-trivial eigenvectors of $\widehat{S}$ are also non-trivial eigenvectors of $\widehat S_{sum}$. However, eigenvectors of $\widehat S_{sum}$ with eigenvalue $1/2$ are in the null space of $\widehat{S}$. In principle, these eigenvectors of $\widehat S_{sum}$ correspond to the individual orthogonal directions in $\Ocal_1$ and $\Ocal_2$, or specifically $\col(\widehat{X}_1) \cap \col(\widehat{X}_2)^\perp$ or $\col(\widehat{X}_1)^\perp \cap \col(\widehat{X}_2)$. Our proposed estimate of joint structure automatically excludes such directions.

 We further provide a more detailed comparison between the proposed approach and the method of \citet{feng2018angle} using angle $\phi =40^\circ$ and varying $\mathrm{SNR}\in\{1, 3, 5\}$. One key advantage of the proposed algorithm is the diagnostic plot, which could be used to assess the quality of joint rank estimation in practice. 
The bottom row of Figure~\ref{fig:ppd-diagnostic-simulations} illustrates these diagnostic plots by showing the spectrum of $P_{\col(\widehat{X}_1)}P_{\col(\widehat{X}_2)}$ together with the proposed bootstrap bound  $1-\widehat{\varepsilon}_1$ (obtained in step 2 of Algorithm~\ref{algo:main}) and noise bound $\lambda_+$ (obtained in step 3 of Algorithm~\ref{algo:main}). Observe that singular values cluster into groups corresponding to joint, non-orthogonal individual and remaining directions. In the high signal-to-noise ratio, our bootstrap bound for joint directions effectively identifies the corresponding top cluster, leading to the correct estimate of the joint rank as 4.
Furthermore, the noise bound and the bootstrap bound do not overlap. In contrast, when signal-to-noise is low, the bootstrap bound becomes loose due to the increased error in subspace estimation (hence inflated $\Delta_k$ in~\eqref{eqn:epsilon_estimates}). Here, the noise bound compensates by filtering out rotated directions since their deteriorated alignment disqualifies them from being included as a joint signal.  In comparison, the top row of Figure~\ref{fig:ppd-diagnostic-simulations} illustrates the spectrum of $P_{\col(\widehat{X}_1)} + P_{\col(\widehat{X}_2)}$ used by AJIVE \citep{feng2018angle}.  The lack of clustering in the spectrum makes the separation of joint directions more challenging. Like our approach, \citet{feng2018angle} uses the maximum of two bounds to determine joint rank: bootstrap estimate of the Wedin bound to identify the joint structure and bootstrap bound to filter out the noise directions. Figure~\ref{fig:ppd-diagnostic-simulations} shows that our joint bound is tighter since the noise bound of \citet{feng2018angle} almost always dominates the joint bound. We believe that the tightness of our joint bound is due to taking explicit advantage of perturbation $\varepsilon_1$ in Theorem~\ref{theorem:main} when performing bootstrap. 
In terms of noise bound, the two approaches agree in values, differing by at most $2\%$ on average across simulations, but the main difference is that our bound $\lambda_+$ is analytical based on random matrix theory whereas \citet{feng2018angle} estimate it using the bootstrap. Overall, Figure~\ref{fig:ppd-diagnostic-simulations} suggests that improved performance of proposed method is due to a more clear separation of joint directions in the spectrum of $P_{\col(\widehat{X}_1)}P_{\col(\widehat{X}_2)}$ compared to the spectrum of $P_{\col(\widehat{X}_1)} + P_{\col(\widehat{X}_2)}$, as well as tighter bound that we use to identify this separation in practice.

\section{Additional experiments}
\label{sec:add_experiments}

\subsection{Simulations on synthetic data for more than two views} 

\begin{table}[t]
\centering
\caption{Quality of the joint and individual subspace estimates in terms of the metric \eqref{eqn:fdr}, where larger values indicate better performance.}
\scalebox{0.8}{
\begin{tabular}{lcccccccccccc}
& \multicolumn{4}{c}{Under-specified rank} & \multicolumn{4}{c}{Estimated rank} & \multicolumn{4}{c}{Over-specified rank} \\

 &  \multicolumn{2}{c}{SNR = 2} & \multicolumn{2}{c}{SNR=0.5} & \multicolumn{2}{c}{SNR = 2} & 
 
 \multicolumn{2}{c}{SNR=0.5} & \multicolumn{2}{c}{SNR = 2} & \multicolumn{2}{c}{SNR=0.5}  \\
 & $90^{\circ}$ & $30^{\circ}$ & $90^{\circ}$ & $30^{\circ}$ & $90^{\circ}$ & $30^{\circ}$ & $90^{\circ}$ & $30^{\circ}$ & $90^{\circ}$ & $30^{\circ}$ & $90^{\circ}$ & $30^{\circ}$  \\

JIV & -- & -- & -- & -- & 9.08 & 7.49 & \textbf{4.81} & \textbf{4.55} & -- & -- & -- & -- \\ \\
AJI & \textbf{8.67} & \textbf{7.77} & \textbf{6.04} & \textbf{5.69} & \underline{9.79} & \underline{8.45} & 3.81 & 3.63 & \underline{9.08} & \underline{7.63} & \textbf{6.21} & \textbf{5.91} \\
SLI$^*$ & -- & -- & -- & -- & 9.47 & 7.43 & 1.17 & 1.09 & -- & -- & -- & -- \\
UNI$^*$ & -- & -- & -- & -- & 9.04 & 8.08 & \underline{4.32} & \underline{3.92} & -- & -- & -- & -- \\
PPD & \underline{8.03} & \underline{7.41} & \underline{5.43} & \underline{5.44} & \textbf{9.83} & \textbf{9.75} & 3.70 & 3.64 & \textbf{9.12} & \textbf{9.07} & \underline{5.67} & \underline{5.57} \\
\end{tabular}
}
\label{table:simulations-multiview}
\end{table}

Similar to Section~\ref{sec:simulation}, we run simulations for the case of 3 views under various rank corruption, signal-to-noise ratio, and angles. We set the joint ranks to be $\rank(J_k) = 3$, and the ranks of the individual subspaces to $\rank (I_1), \rank(I_2), \rank(I_3)$. We take the matrix dimensions to be $n=35$ and $p_1 = 40, p_2=45, p_3=50$. We report our results in Table~\ref{table:simulations-multiview}. From the table, we see that our method consistently performs in the top among other methods. In addition, we note that the proposed approach offers strong performance even under the more difficult conditions such as low signal-to-noise ratio and highly rotated angles. 

\subsection{Ablation study for bootstrap}
\label{sec:bootstrap-ablation}
\begin{figure}[!t]
    \centering
    \includegraphics[width=\linewidth]{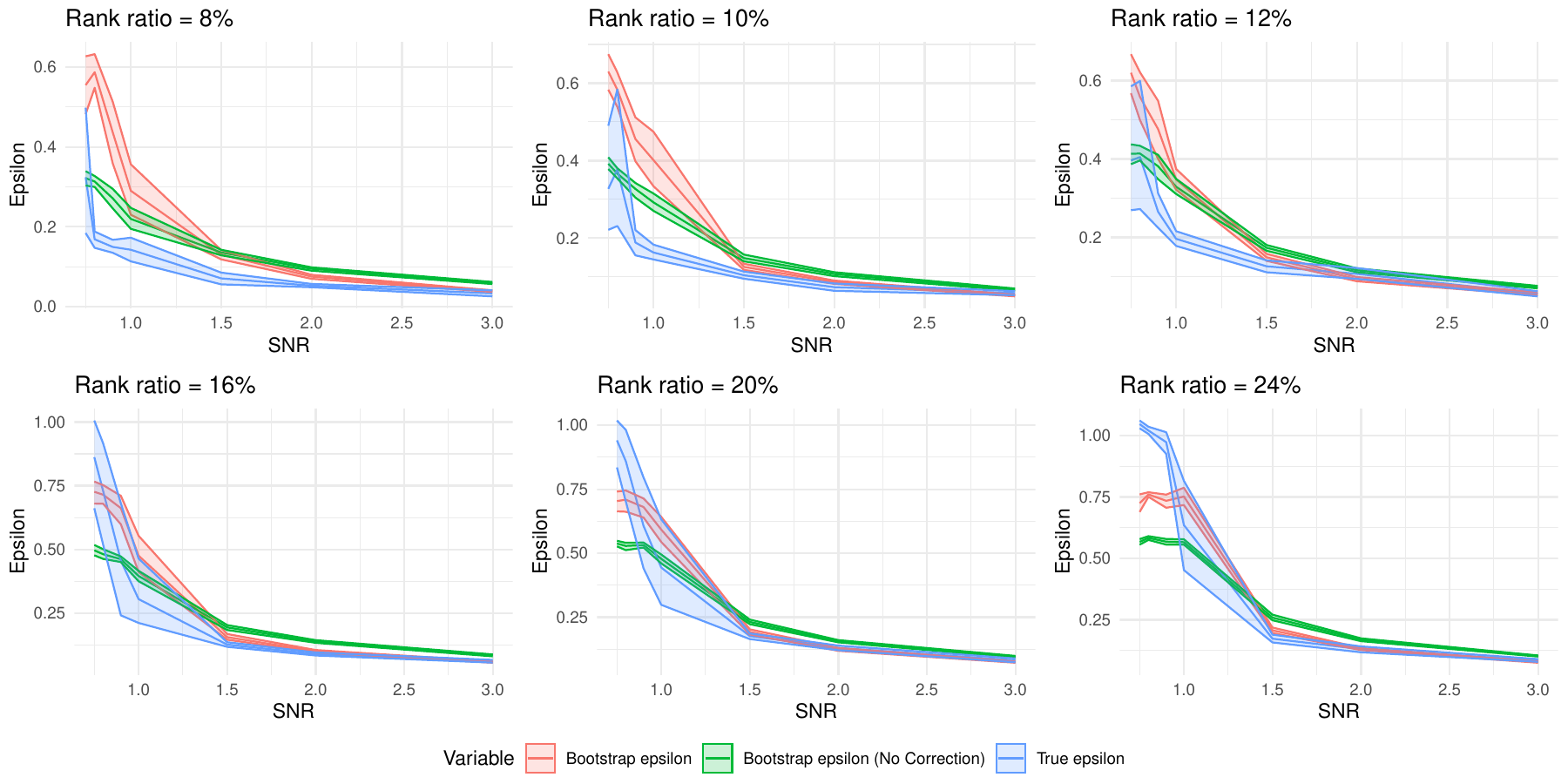}
    \caption{Estimate of $\varepsilon_1$ computed with and without our modification, aligning the view-specific bases. We simulate the data as in Section~\ref{sec:simulation}, fix the angle between individual subspaces at $60^\circ$, while varying the rank-to-dimension and signal-to-noise ratios. We produce the confidence intervals based on 10 re-runs.}
    \label{fig:bootstrap-comparison-naive}
\end{figure}
We modify the bootstrap method proposed in \citet{prothero2024data} to better align the resampled singular vector bases with the observed data. Specifically, given orthogonal resampled bases \( U_{1,(b)} \) and \( U_{2,(b)} \), we adjust \( U_{2,(b)} \) to align with \( U_{1,(b)} \) using the formula:$
U_{2,(b)} \leftarrow U_{1,(b)} \cos (\Sigma_{\widehat{M}}) + U_{2,(b)} \sin (\Sigma_{\widehat{M}}),$
where \( \Sigma_{\widehat{M}} \) is a diagonal matrix containing the principal angles between the column spaces \( \operatorname{col}(\widehat{X}_1) \) and \( \operatorname{col}(\widehat{X}_2) \). This alignment ensures that the resampled bases reflect the geometric relationship observed in the data, which is crucial for accurate estimation of \( \varepsilon_1 \) and $\varepsilon_2$. Without this correction, a straightforward approach would be to randomly sample \( U_{1,(b)} \) and \( U_{2,(b)} \) from the Haar measure. However, as we argue in Section~\ref{sec:estimation}, such an approach tends to underestimate \( \varepsilon_1 \) and $\varepsilon_2$, particularly when the rank-to-dimension ratio is low. In this regime, randomly sampled bases are likely to be nearly orthogonal, which does not capture the true alignment between the column spaces of \( \widehat{X}_1 \) and \( \widehat{X}_2 \).

We investigate the effect of our correction in Figure~\ref{fig:bootstrap-comparison-naive} for $\varepsilon_1$ with similar conclusions holding for $\varepsilon_2$. To generate the figure, we simulate data following the procedure in Section~\ref{sec:simulation}, fixing the angle between individuals at \( 60^\circ \) while varying the rank-to-dimension and signal-to-noise ratios. The 95\% confidence intervals are computed using 10 re-runs.

From the figure, we observe that our corrected bound provides better control over the true \( \varepsilon_1 \), offering more accurate coverage in the low signal-to-noise ratio regime and tighter bounds in the high signal-to-noise ratio regime. Recall that $\varepsilon_1 := \left\| P_{\operatorname{col}(X_1)} \left( \Delta_1 + \Delta_2 + \Delta_1 \Delta_2 \right) P_{\operatorname{col}(X_2)} \right\|_2$. The projection operators \( P_{\operatorname{col}(X_1)} \) and \( P_{\operatorname{col}(X_2)} \) reduce the estimate by filtering out components of the error terms that lie within the column spaces of \( X_1 \) and \( X_2 \), respectively. Recall from Figure~\ref{fig:noisespectrum}, for random projections \( P_{\operatorname{col}(X_1)} \) and \( P_{\operatorname{col}(X_2)} \), the spectral density depends on the rank-to-dimension ratios: the bulk of singular values is around zero, with the tail extending towards one as the rank-to-dimension ratio increases. Hence, the naive estimate tends to underestimate \( \varepsilon_1 \), especially when the rank-to-dimension ratio is low, because the resampled projections \( P_{\operatorname{col}(X_1)} \) and \( P_{\operatorname{col}(X_2)} \) are nearly orthogonal, causing more directions to be filtered out. This effect is more pronounced for low signal-to-noise ratios, as the error terms $\Delta_k$ are larger. We note that our corrected bound tends to be tighter in the high signal-to-noise ratio regime, and we leave a detailed investigation of this phenomenon for future work.

\subsection{Empirical error of the bootstrap}
\label{sec:bootstrap-error}
\begin{figure}[!t]
    \centering
    \includegraphics[width=0.9\linewidth]{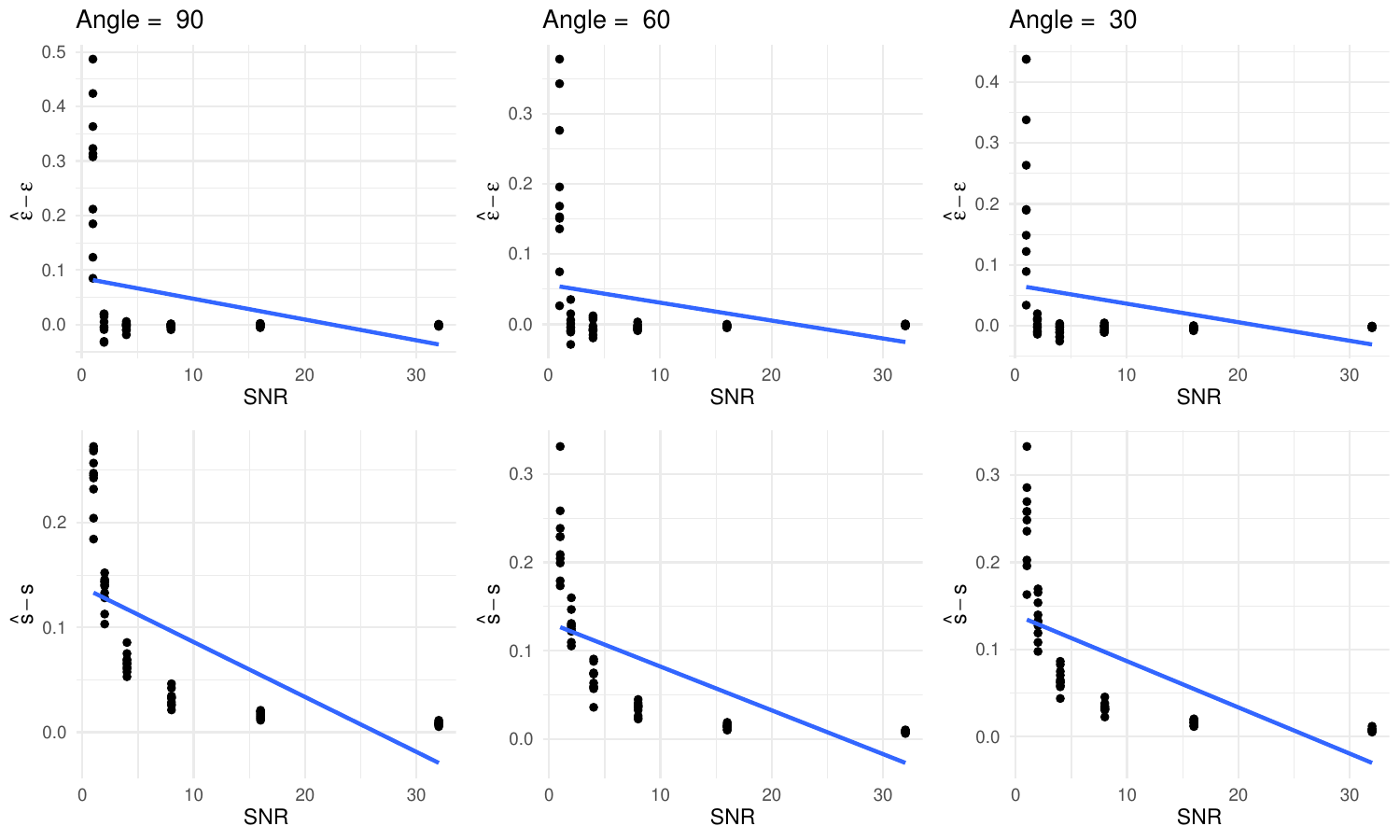}
    \caption{\textbf{Top row: }Error in the bootstrap estimate of $\epsilon_1$. \textbf{Bottom row:} error in the noise variance estimation by the robust estimator $y_{\text{med}} / \mu_\beta$.}
    \label{fig:bootstrap-error}
\end{figure}

We simulate the data following Section~\ref{sec:simulation}, setting. We run our bootstrap model for a range of signal-to-noise ratios and angles, supplying the model with the true marginal ranks. We report the difference $\hat \epsilon_1 - \epsilon_1$ in Figure~\ref{fig:bootstrap-error}. From the figure, we see that our bootstrap estimator tends to be accurate, except for the low signal-to-noise ratios for which it overestimates the true parameter. We investigate this issue and find it related to the overestimation of the noise variance by the robust estimator $y_{\text{med}} / \mu_\beta$ with $y_{\text{med}}, \mu_\beta$ denoting the median singular value of $Y$ and the median of the Marchenko-Pastur distribution with parameter $\beta$. We show the error in the variance estimation in Figure~\ref{fig:bootstrap-error}. Recall that in our bootstrap as described in Section~\ref{subsubsec:bootstrap}, we impute the signal directions back into the noise estimate $\check{E}_k$, which are scaled by the noise variance. As our estimate of the noise variance overestimates the truth, correspondingly our noise estimate becomes inflated, which finally leads to the inflated $\widehat{\varepsilon}_1$ estimate. We leave the improvement of our bootstrap estimator in the low signal-to-noise ratio regime as future work.

\subsection{Asymptotic distribution of spectrum of random projections}
\label{sec:noise-plot-appendix}

\begin{figure}[!t]
    \centering\includegraphics[width=0.9\linewidth]{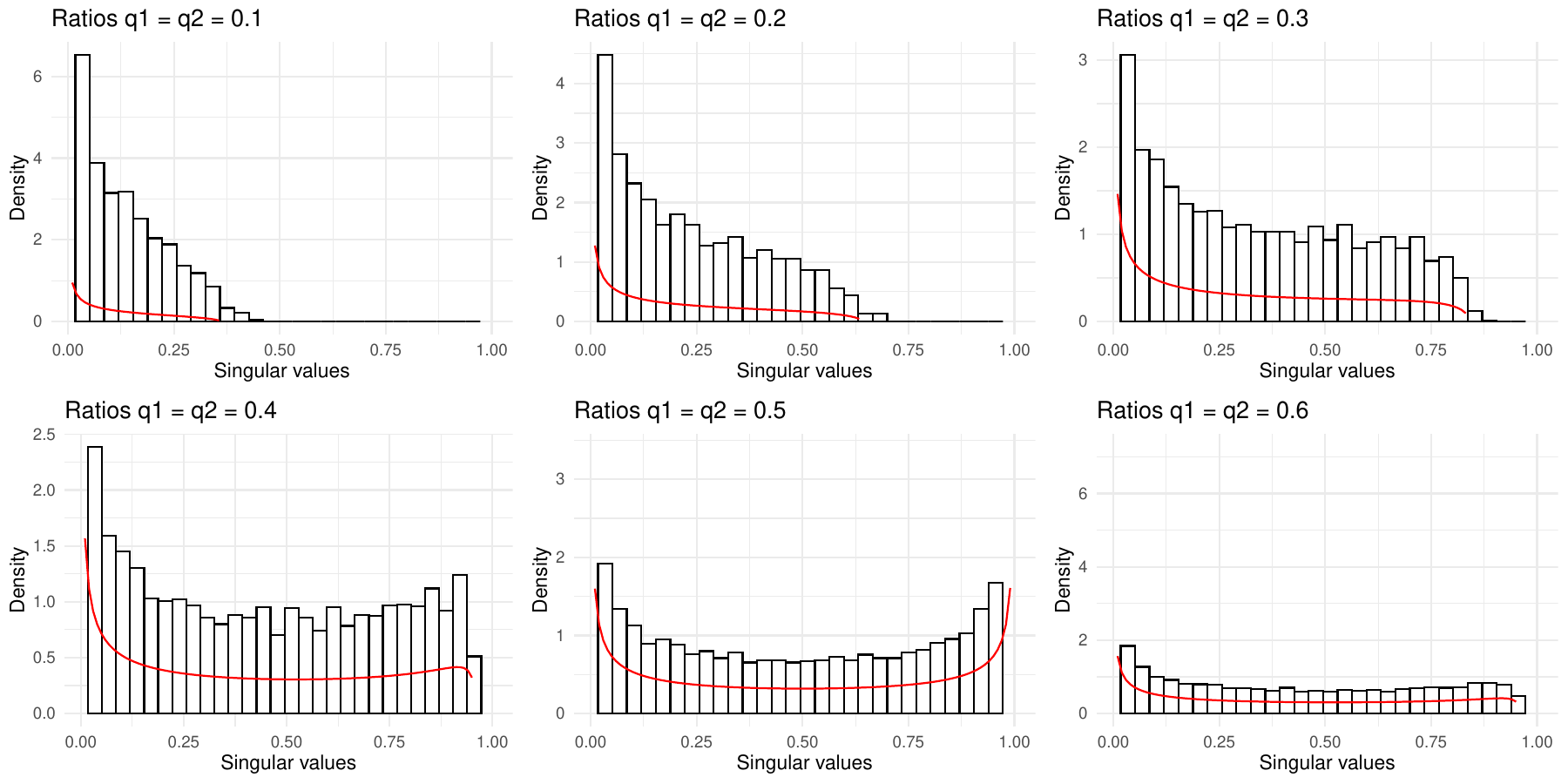}
    \caption{Asymptotic distribution of the spectrum of the product of two random projections, parametrized by the subspace-to-ambient dimension ratios $q_1=r_1/n$ and $q_2=r_2/n$.}
    \label{fig:noisespectrum}
\end{figure}

To simulate the spectrum of two random projection matrices, we sample orthonormal bases \( U_1 \) and \( U_2 \) of ranks $r_1$ and $r_2$ from the Haar measure. We set the ambient dimension to \( n = 100 \). We then compute the squares of the singular values of the matrix \( U_1^\top U_2 \). In Figure~\ref{fig:noisespectrum}, we present the histogram of these squares, $\lambda$, overlaid with the continuous part of the asymptotic density provided in Equation~\eqref{eqn:noise_bound}.

\end{document}